%% file: example_paper.tex
\documentclass{article}

\usepackage{microtype}
\usepackage{graphicx}
\usepackage{subcaption}
\usepackage{booktabs} % for professional tables

\usepackage{hyperref}

\usepackage[accepted]{icml2026}

\usepackage{amsmath}
\usepackage{amssymb}
\usepackage{mathtools}
\usepackage{amsthm}
\usepackage{listings}
\usepackage{threeparttable}
\usepackage{subcaption}
\usepackage{tikz}
\usepackage{arydshln}
\usetikzlibrary{tikzmark,decorations.pathreplacing}
\usetikzlibrary{calc}
\usepackage{enumitem}
\usepackage{algorithm}
\usepackage{algorithmic}
\usepackage{booktabs}
\usepackage{threeparttable}
\usepackage{float}
\usepackage{pgfplots}
\pgfplotsset{compat=1.18}
\usepackage{pgfplotstable}
\pgfplotsset{compat=1.18}
\usepackage{subcaption}
\usepackage{tikz}
\usepackage{arydshln}
\usetikzlibrary{tikzmark,decorations.pathreplacing}
\usetikzlibrary{calc}
\usepackage{enumitem}
\usepackage{newfloat}
\usepackage{listings}
\usepackage[capitalize,noabbrev]{cleveref}

\theoremstyle{plain}
\newtheorem{theorem}{Theorem}[section]

\theoremstyle{definition}

\theoremstyle{remark}

\usepackage[textsize=tiny]{todonotes}

\icmltitlerunning{Submission and Formatting Instructions for ICML 2026}

\begin{document}

\twocolumn[
 % \icmltitle{OneStep-GNF: Direct One-Step Graph Neural Flows for  Interaction\\ Modeling in Irregular Time Series Classification}
  % \icmltitle{Removing Iteration in Interactions: One-Step Graph-Structured Neural Flows for  Irregular Time Series Classification}
    \icmltitle{One-Step Graph-Structured Neural Flows for 
    Irregular Multivariate \\
    Time Series Classification}



  \icmlsetsymbol{equal}{*}

  \begin{icmlauthorlist}
    \icmlauthor{Mengzhou Gao}{yyy}
    \icmlauthor{Kaiwei Wang}{yyy}
    \icmlauthor{Pengfei Jiao}{yyy}

  \end{icmlauthorlist}

  \icmlaffiliation{yyy}{School of Cyberspace, Hangzhou Dianzi University}

  \icmlcorrespondingauthor{Pengfei Jiao}{pjiao@hdu.edu.cn}

  \icmlkeywords{Machine Learning, ICML}

  \vskip 0.3in
]



\printAffiliationsAndNotice{}  % no special notice (required even if empty)

\begin{abstract}
Neural Flows efficiently model irregular multivariate time series by directly learning ODE solution trajectories with neural networks, bypassing step-by-step numerical solvers. Despite their efficiency, many existing approaches treat variables independently, leaving inter-variable interactions underexplored. Moreover, their one-step mapping makes interaction modeling inherently challenging, as it removes the iterative refinement of interactions during learning. To address this challenge, we propose one-step Graph-Structured Neural Flows (GSNF), which introduce two auxiliary-trajectory self-supervision strategies to strengthen interaction learning: (i) interaction-aware trajectory generation via re-initialization, which induces trajectory divergence to expose graph-induced interactions, with a theoretically derived lower bound on divergence; and (ii) reverse-time trajectory generation, which enforces forward–backward consistency to regularize graph learning, enabled by flow invertibility. Experiments on five real-world datasets show that GSNF achieves state-of-the-art classification performance with highly competitive training time and memory usage.

\end{abstract}

% \begin{figure}[t]
% \centering
% \begin{subfigure}[t]{0.45\columnwidth}
%     \centering
%     \includegraphics[width=\linewidth]{figure/toy1.png}
%     \caption{without interaction}
%     \label{fig:graph_with_repulsion}
% \end{subfigure}
% \hfill
% \begin{subfigure}[t]{0.45\columnwidth}
%     \centering
%     \includegraphics[width=\linewidth]{figure/toy2.png}
%     \caption{with interaction}
%     \label{fig:graph_without_repulsion}
% \end{subfigure}
% \caption{The sensitivity of flow dynamics to initial conditions re-
% sults in trajectory divergence that reveals latent interactions.}
% \label{fig:graph_comparison}
% \end{figure}

\section{Introduction}
\begin{figure}[t]
\centering
\includegraphics[width=1\linewidth,]{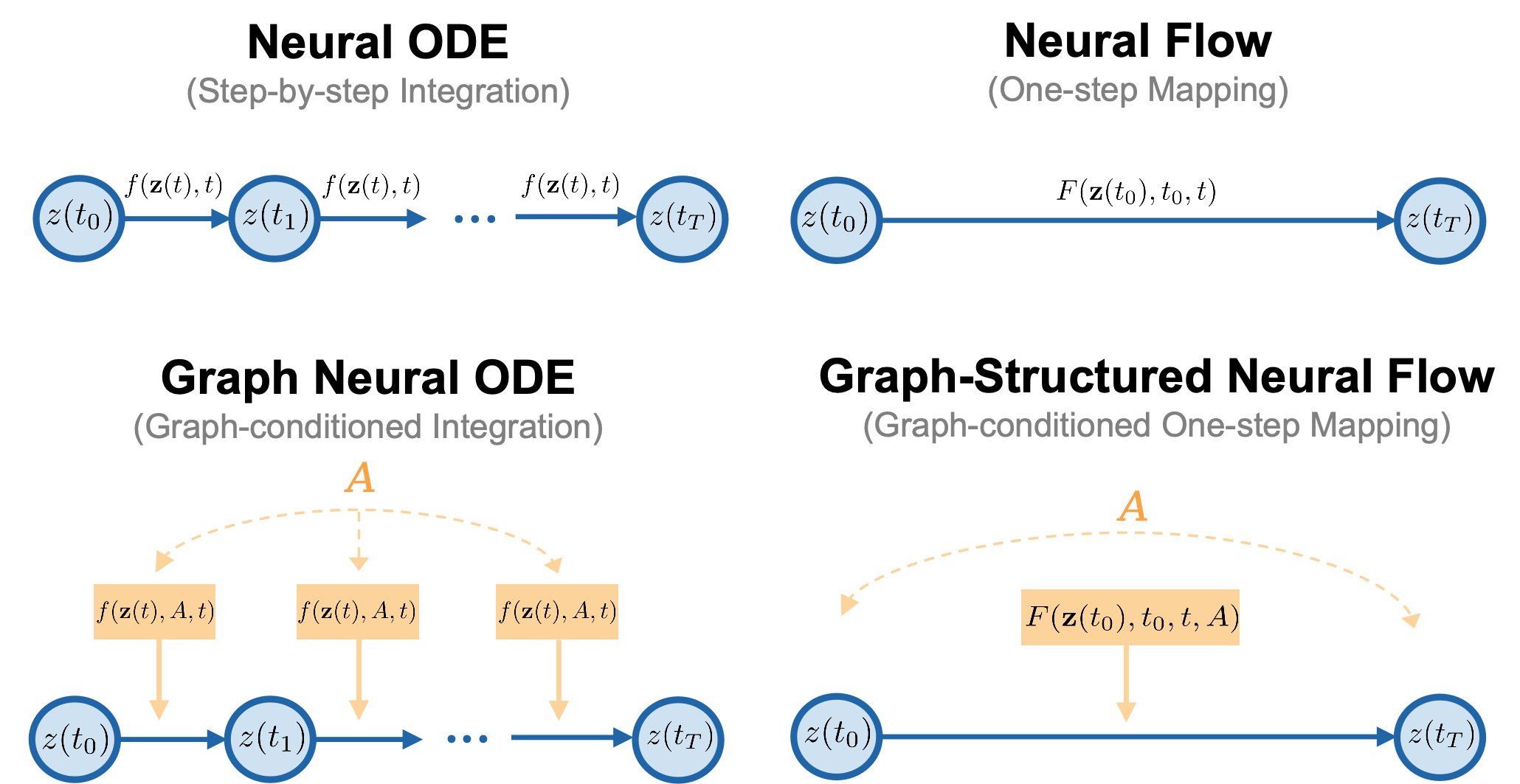}
\caption{Comparison of continuous-time models. Neural ODEs evolve latent states via step-by-step numerical integration, whereas Neural Flows perform one-step mappings.
Graph Neural ODEs incorporate interaction modeling  by conditioning dynamics on an interaction graph, while Graph-Structured Neural Flows integrate graph-conditioned interactions into one-step mappings.}
\label{fig:intro}
\end{figure}

Irregular multivariate time series arise from event-driven and heterogeneous data acquisition processes, leading to non-uniform sampling intervals and missing observations, which challenge conventional time series analysis methods. ~\cite{xiao2024ivp} Despite such irregularity, the underlying system dynamics are typically assumed to evolve continuously over time, making continuous-time models a natural framework for analyzing such data.~\cite{chen2023contiformer}

Existing continuous-time models for irregular multivariate time series can be broadly categorized into implicit and explicit formulations.~\cite{oh2025dualdynamics} Implicit methods, including Neural ODEs~\cite{chen2018neural}, Neural CDEs~\cite{kidger2020neural}, and Neural SDEs~\cite{kidger2021efficient}, define latent dynamics through differential equations and rely on numerical solvers to evolve system states.  While this formulation offers flexibility in handling irregular observations, the reliance on  step-by-step numerical integration  introduces  non-negligible computational and memory overhead~\cite{oh2025comprehensive}. In contrast, explicit approaches such as Neural Flows~\cite{bilovs2021neural} learn a one-step mapping that directly models ODE solution trajectories with neural networks, thereby avoiding iterative integration and enabling more efficient computation.

%IVP-VAE

While continuous-time models effectively handle irregular observations, many existing approaches treat variables independently, leaving inter-variable interactions underexplored in multivariate settings ~\cite{mercatali2024graph}.
Within the Neural ODE framework, 
this gap has been partially addressed by incorporating graph structures into continuous-time dynamics, enabling joint modeling of multiple interacting time series. For example, Graph Neural ODEs~\cite{poli2019graph} parameterize ODE vector fields with graph neural networks, introducing relational inductive biases into continuous-time dynamics.

However, incorporating graph-structured interactions into Neural Flow 
in a way that directly governs trajectory evolution remains largely unexplored. By collapsing continuous-time dynamics into a one-step mapping, 
Neural Flows remove numerical solvers and their intermediate latent states, causing interaction
effects to be applied only once rather than  iteratively refined along the trajectory. This limits opportunities for incremental correction and weakens the stabilization provided by iterative interaction updates in solver-based models. This challenge is further amplified when graph structures must be learned under missing observations, where interaction learning and trajectory prediction are jointly resolved within a one-step formulation.

To address this challenge, we propose \textbf{G}raph-\textbf{S}tructured \textbf{N}eural \textbf{F}lows (GSNF), a framework that strengthens interaction learning in Neural Flows while retaining the efficiency of one-step mappings. GSNF introduces auxiliary trajectory-level self-supervision to compensate for the absence of iterative interaction refinement. These auxiliary signals encourage interaction-induced trajectory divergence  and forward–backward temporal consistency, and can be incorporated into a VAE-based framework with parallel computation, enabling efficient end-to-end optimization of latent states and graph structures. Our main contributions are summarized as follows:

%through two auxiliary-trajectory self-supervision strategies applied during the generation stage. Specifically, GSNF introduces interaction-aware trajectory generation (ITG), which induces controlled trajectory divergence via re-initialization to expose latent interactions, for which we theoretically establish a lower bound on the induced divergence.   In addition, GSNF leverages the invertibility of Neural Flows to perform reverse-time trajectory generation (RTG), enforcing forward–backward consistency as a complementary regularization signal.  Both self-supervision strategies are seamlessly integrated into a VAE-based framework with parallel computation, enabling efficient end-to-end optimization of latent states and graph structures. Our main contributions are summarized as follows:
\begin{itemize}
%\item We propose GSNF for interaction-aware modeling of irregular time series, integrating graph-structured interactions into a one-step Neural Flow framework while preserving efficiency and improving accuracy.
\item  We propose GSNF, a graph-structured Neural Flow that intrinsically embeds inter-variable interactions within a one-step formulation.

\item We design auxiliary trajectory-level self-supervision for GSNF, including interaction-aware trajectory generation (ITG) and reverse-time trajectory generation (RTG), with a theoretical lower bound on the divergence induced by ITG.
%\item We introduce auxiliary trajectory-level self-supervision in GSNF, including interaction-aware trajectory generation and reverse-time trajectory generation. A theoretically derived lower bound is established on the divergence induced by interaction-aware trajectories.
\item Experiments on five real-world datasets show that GSNF achieves state-of-the-art performance and is the most efficient among top-performing methods.
\end{itemize}

\begin{figure*}[t]
\centering
\hspace*{20pt}
\includegraphics[width=1\textwidth]{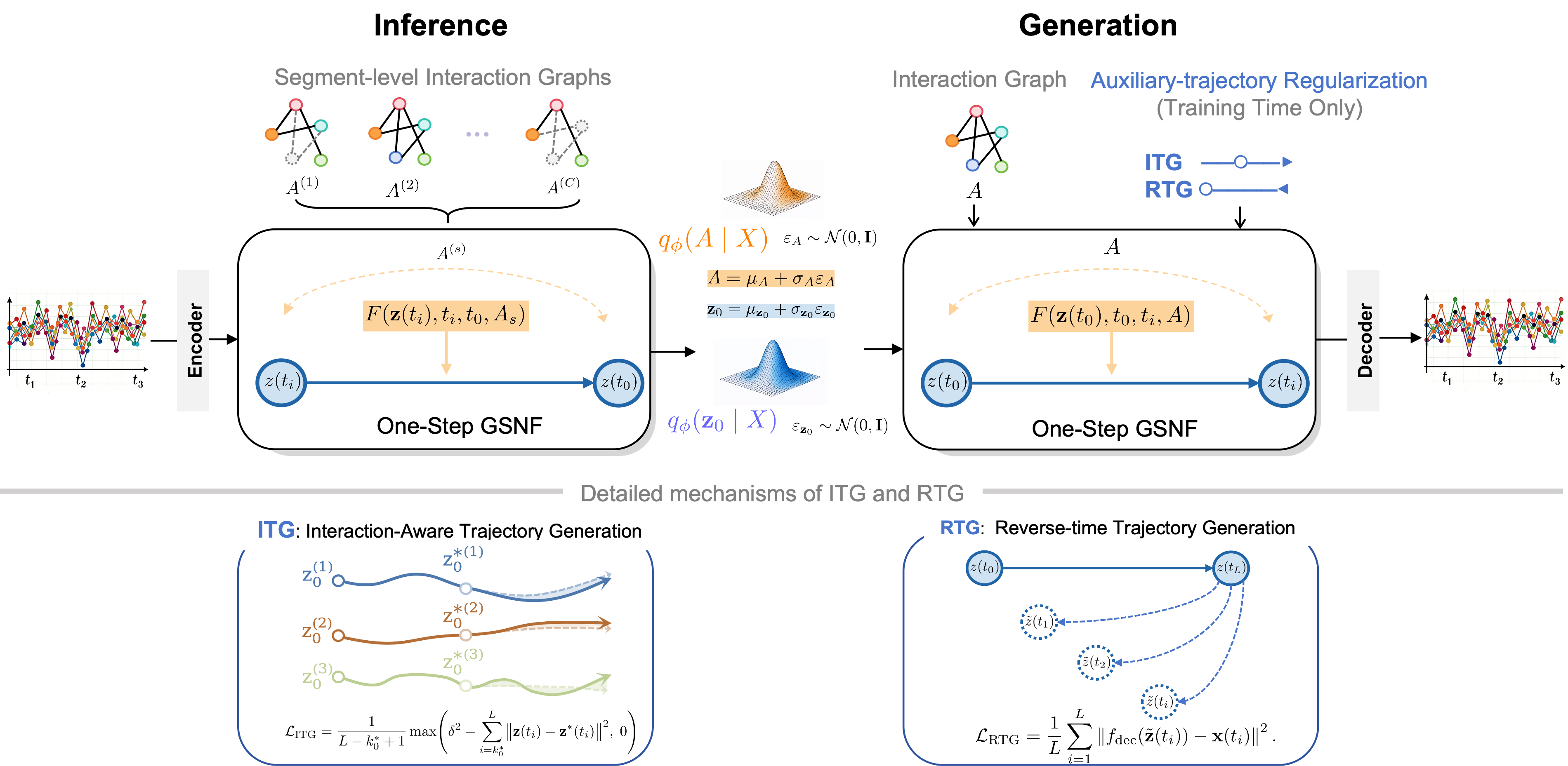} % Reduce the figure size so that it is slightly narrower than the column.
\caption{Overall framework.  \textbf{Top}:Inference and generation with one-step GSNF. \emph{During inference}, segment-level interaction graphs $\{A^{(s)}\}_{s=1}^C$ are aggregated into a posterior $q_\phi(A\mid X)$, and latent states are propagated backward through GSNF, conditioned on their corresponding segment-level graphs to infer $q_\phi(\mathbf{z}_0\mid X)$.   \emph{During generation}, samples of $A$ and $\mathbf{z}_0$ are used by GSNF to compute latent states forward at arbitrary time points via one-step mappings, with auxiliary trajectory-level regularization applied during training. \textbf{Bottom}: Detailed mechanisms of ITG and RTG. \emph{ITG} induces divergence between the original trajectory $\mathbf z(t)$ and the re-initialized trajectory $\mathbf z^\ast(t)$ via intermediate-time re-initialization, regularized by the trajectory-level loss $\mathcal L_{\mathrm{ITG}}$.
\emph{RTG} enforces forward–backward consistency between the forward trajectory $\mathbf z(t)$ and the reverse-time trajectory $\tilde{\mathbf z}(t)$ through the reconstruction loss $\mathcal L_{\mathrm{RTG}}$.
}
\label{fig:framework}
\end{figure*}

\section{Related work}
\subsection{Continuous-Time Models for Irregular Time Series}
Modeling irregularly sampled time series has long been a challenging problem, primarily due to non-uniform observation intervals and asynchronous measurements. Neural Ordinary Differential Equations (Neural ODEs)~\cite{chen2018neural} first introduced a continuous-depth framework for modeling such data. Latent ODE and GRU-ODE-Bayes~\cite{de2019gru} extended Neural ODEs by introducing discontinuities at observation points, thereby improving predictive accuracy on irregularly sampled sequences. However, these methods rely on numerical solvers, which limits computational efficiency.

Neural Flows~\cite{bilovs2021neural} addressed this limitation by directly parameterizing solution trajectories with neural networks, removing the dependency on ODE solvers and enabling one-step trajectory computation. Building on this idea, IVP-VAE~\cite{xiao2024ivp} employed a single invertible initial value problem to achieve parallel continuous-time modeling, offering better scalability and efficiency compared to RNN-based ODE methods. DualDynamics~\cite{oh2025dualdynamics} further combined Neural Differential Equation (NDE)-based solvers with Neural Flows to enhance the expressive power of continuous models. Despite these advances, existing methods lack explicit modeling of inter-variable dependencies that drive trajectory evolution.

\subsection{Graph Learning in Continuous-Time Models}
In multivariate time series, the complex interdependencies among variables have motivated the integration of graph structures into continuous-time modeling frameworks. LG-ODE~\cite{huang2020learning} was one of the earliest attempts to incorporate graph neural networks (GNNs) into Neural ODEs, enabling neighborhood information aggregation to improve variable modeling. Coupled Graph ODE~\cite{huang2021coupled} extended this idea to edge-level dynamics, while GG-ODE~\cite{huang2023generalizing} further generalized it to multi-environment settings. However, these methods rely on multi-step numerical integration to propagate information across the graph, resulting in high computational costs and incompatibility with Neural Flow models that emphasize efficient one-step computation.

Alternative approaches have attempted to infer graph structures directly from data. For instance, Raindrop~\cite{zhang2021graph} learns a dependency graph by averaging time-varying attention weights across timestamps and pruning weak connections. Similarly, GNeuralFlow~\cite{mercatali2024graph} employs a directed acyclic graph (DAG) to represent conditional dependencies among variables and jointly learns this structure alongside Neural Flows, offering flexibility for modeling unknown interaction graphs. Notably, the graph in GNeuralFlow is only used to parameterize the initial condition, and does not directly participate in the subsequent flow-based trajectory evolution.

\section{Preliminaries}

\subsection{Neural Flows} 

% \paragraph{Neural Ordinary Differential Equations.} Neural ODEs\cite{chen2018neural} model the system in Eq.~\eqref{eq:autonomous_system} governed by a parameterized vector field
% \begin{equation}
%     \dot{\mathbf{z}}(t)=f(\mathbf{z}(t), t; \theta), 
% \end{equation}
% and compute trajectories via numerical ODE solvers.

% %Time Dependence in Non-Autonomous Neural ODEs
% The corresponding flow operator defined by,
% \begin{equation}
%     \phi_t(\mathbf{z}_0; \theta)=\mathbf{z}_t(\theta),
% \end{equation}
% is a parametric map from $\mathbb{R}^n\mapsto\mathbb{R}^n$.

Neural Flows~\cite{bilovs2021neural} directly model the solution trajectory of an ODE  with a neural network,
%Neural Flows~\cite{bilovs2021neural} directly model the solution trajectory of the system in Eq.~\eqref{eq:autonomous_system}  with a neural network,
\begin{equation}
   \label{eq:neuralflow}
    \mathbf{z}(t)=F(\mathbf{z}_0, t_0, t), \quad \mathbf{z}(t_0)=\mathbf{z}_0,
\end{equation}
instead of parameterizing the derivative as in Neural ODEs~\cite{chen2018neural}.  This eliminates the need for numerical ODE solvers, providing a faster alternative to the ODE.  To guarantee correspondence with an underlying ODE, $F$ must (i) satisfy the initial condition $F(\mathbf{z}_0, t_0, t_0 )=\mathbf{z}_0$, anchoring the trajectory to the specified starting point; and (ii) be invertible in $\mathbf{z}$  for all $t$, ensuring uniqueness of trajectories without self-intersections.

\subsection{Problem Formulation}
Consider a labeled multivariate time series $(X, y)$, where $y$ denotes the categorical label and $X=\{(\mathbf{x}_k, t_k, \mathbf{m}_k)\}_{k=1}^{L}$ is a sequence of $L$ irregularly sampled observations. Each observation $\mathbf{x}_k\in\mathbb{R}^{D_x}$ is recorded at non-uniform time points $t_k$ and is associated with a binary mask $\mathbf{m}_k\in\{0,1\}^{D_x}$, where $m_{k, i}=1$ denotes that the $i$-th variable is observed at time $t_k$ and $m_{k, i}=0$ indicates a missing value.  The dataset is given by $\mathcal{X}=\{(X_n, y_n)\}_{n=1}^N$, consisting of $N$ such sequences.

We adopt a latent continuous-time modeling framework in which the observed
sequence $X$ is assumed to arise from interacting latent dynamics
$\mathbf{z}(t)\in\mathbb{R}^{D_z}$ governed by an interaction structure $A$.  Our goal is to jointly infer the latent initial state $\mathbf{z}_0$ and the
interaction structure $A$, and to integrate them into the Neural Flow for
downstream classification.
Specifically, we aim to  learn a classifier $f: X \rightarrow \mathcal{Y}$ such that
the predicted label
\[
\hat{y} = f\bigl(\mathbf{z}_0, A, F(\mathbf{z}_0, A, t)\bigr)
\]
matches the ground-truth label $y$ as accurately as possible.

\section {Graph-Structured Neural Flows} 
To model continuous-time dynamics with explicit interaction structure, we propose \textit{Graph-Structured Neural Flows} (GSNF) as the core dynamical module of our framework.  Unlike standard neural flows~\cite{bilovs2021neural} that evolve latent states independently,
GSNF conditions the flow on a interaction graph $A$, allowing interactions to influence the system evolution throughout the flow.

Formally, GSNF defines a Neural Flow that maps an initial latent state at time $t_0$
to a future state at time $t$,
\begin{equation}
\mathbf{z}(t) = F(\mathbf{z}(t_0), t_0, t, A).
\label{eq:flow_def}
\end{equation}

In this work, we implement GSNF following the ResNet flow architecture~\cite{bilovs2021neural}.
Unlike GNeuralFlow~\cite{mercatali2024graph}, which incorporates the graph only to initialize the latent state,
GSNF embeds graph-structured interactions directly into the flow dynamics, allowing interactions to
modulate state evolution throughout the trajectory.
Specifically, the flow mapping is defined as
\begin{equation}
\begin{aligned}
F(\mathbf{z}(t_0), t_0, t, A)
&=
\mathbf{z}(t_0)
+
\varphi(t - t_0)\, g(\mathbf{z}(t_0), t_0, t, A), \\
g(\mathbf{z}(t_0), t_0, t, A)
&=
\mathrm{MLP}(\mathbf{z}(t_0)||t|| t_0)\\
&\qquad\odot\;
\mathrm{GCN}(A, \mathbf{z}(t_0)||t||t_0),
\end{aligned}
\label{eq:gsnf}
\end{equation}
where $\varphi(\cdot)$ satisfies $\varphi(0)=0$, ensuring the initial condition $F(\mathbf{z}(t_0), t_0, t_0, A)=\mathbf{z}(t_0)$, and is instantiated with a
TimeFourier encoding. The operator $\odot$ denotes the element-wise product, and $||$denotes row-wise concatenation. 

A key property of GSNF is its invertibility, which enables well-defined forward and reverse mappings. We provide sufficient conditions in the following theorem.

\begin{theorem}[Invertibility of Graph-Structured Neural Flows]
\label{thm:inveritbility}
The GSNF defined in \eqref{eq:gsnf} is invertible if $\varphi(\cdot)\in[0,1)$ and $g(\cdot, t_0, t, A)$ is a contractive mapping, which can be ensured
by applying spectral normalization to all linear layers in both the MLP and the GCN.
\end{theorem}

\section{Overall Framework}
Building upon the previously introduced GSNF, our framework (Fig.~\ref{fig:framework}) adopts a latent-variable (VAE) structure to jointly infer latent initial states $\mathbf{z}_0$ and the interaction graph $A$ from irregularly sampled multivariate time series.  The inferred states $\mathbf{z}_0$ and graph $A$ are then used to evolve latent trajectories  in \textit{one-step via GSNF}, capturing interaction-aware dynamics while enabling parallel computation and avoiding iterative updates.  To enforce trajectory consistency and integrate the interaction graph effectively, we further propose a \textit{self-supervised trajectory generation} mechanism, which leverages both re-initialization and reversible trajectory constraints (Section \ref{sec:self_supervised}). The following subsections describe each component in detail.

\subsection{Initial Latent Condition Inference}
Since exact posterior inference of the interaction graph $A$ and the initial latent state $\mathbf{z}_0$ is intractable~\cite{kingma2013auto}, we approximate them using variational distributions  $q_\phi(A\mid X)$ and $q_\phi(\mathbf{z}_0\mid X)$ as described below. During inference, missing observations are zero-imputed.

\paragraph{Posterior of the interaction graph.} 
% \begin{figure}[t]
% \centering
% \includegraphics[width=0.9\linewidth,]{figure/graph.png}
% \caption{interaction graph}
% \label{fig:graph}
% \end{figure}

To roubustly infer interaction structure under irregular missing observations, we perform graph inference over local temporal segments and aggregate segment-level posteriors into a global interaction graph. Specifically, segment-level adjacency matrices $\{A^{(s)}\}_{s=1}^C$
are introduced during inference to capture locally reliable interactions
and are marginalized into a single posterior $q_\phi(A\mid X)$,
which defines the time-invariant interaction graph used by the generative model.

\emph{(i) Segment-level  representation.} Let $\{X^{(s)}\}_{s=1}^C$ denote $C$ consecutive local temporal segments of  length $M=L/C$. We compute a segment-level representation by temporal averaging as  $\bar{\mathbf{x}}^{(s)} = \frac{1}{M} \sum_{k=1}^{M} \mathbf{x}_{(s-1)M+k}$.  

\emph{(ii) Segment-level adjacency inference.}
Variable interactions are inferred via self-attention~\cite{zhou2023detecting}, with query and key vectors $\mathbf{q}_i^{(s)} = \bar{\mathbf{x}}_i^{(s)} W^Q$ and $\mathbf{k}_i^{(s)} = \bar{\mathbf{x}}_i^{(s)} W^K$ for variable $i$, where $\bar{x}_i^{(s)} \in \mathbb{R}$ denotes the segment-level average
of variable $i$, and $W^Q, W^K \in \mathbb{R}^{1 \times D_x}$ are learnable linear projection matrices. The adjacency weight from variable $i$ to $j$ is computed as
$
a_{ij}^{(s)} = \frac{\exp(\mathbf{q}_i^{(s)} (\mathbf{k}_j^{(s)})^\top / \sqrt{D_x})}{\sum_{j'=1}^{D_x} \exp(\mathbf{q}_i^{(s)} (\mathbf{k}_{j'}^{(s)})^\top / \sqrt{D_x})},
$
forming the \textit{segment-level} adjacency matrix $A^{(s)} = [a_{ij}^{(s)}]$.

\emph{(iii) Posterior aggregation.} 
Each $A^{(s)}$ parameterizes a Gaussian variational component 
$
q_\phi(A^{(s)} \mid X^{(s)}) = \mathcal{N}(\mu_{A^{(s)}}, \sigma_{A^{(s)}})
$, where $\mu_{A^{(s)}} = h_A(A^{(s)})$ and $\sigma_{A^{(s)}} = \text{Softplus}(h_A(A^{(s)}))$, with $h_A(\cdot)$ denoting a shared feed-forward network. Gaussian samples are normalized via softmax to obtain valid adjacency matrices. The posterior over the full interaction graph is a weighted mixture of segment-level posteriors:

\begin{equation}
\label{eq:posterior_A}
  q_\phi(A\mid X) = \sum_{s=1}^C w_s \cdot q_\phi(A^{(s)} \mid X^{(s)}),
\end{equation}
with mixture weights 
$
   w_s = \frac{D_{KL}(q_\phi(A^{(s)} \mid X^{(s)})\,\|\,p(A))}{\sum_{j=1}^{C} D_{KL}(q_\phi(A^{(j)} \mid X^{(j)})\,\|\,p(A))},
$
where  $p(A)$ denotes a standard Gaussian prior. These weights assign higher importance to segments that induce more informative posteriors, thereby mitigating the effect of irregular missingness.

\paragraph{Posterior of the Initial Latent State.}At each time point $t_k$, the observation $\mathbf{x}_k$ and its mask $\mathbf{m}_k$ are concatenated and encoded  into a latent representation
$\mathbf{z}(t_k) = f_{\mathrm{enc}}(\mathbf{x}_k \mid \mathbf{m}_k)$. To infer the initial latent state $\mathbf{z}_0$, 
we follow the IVP-VAE paradigm~\cite{xiao2024ivp} and treat each latent state $(\mathbf{z}(t_k), t_k)$ as an initial condition of the underlying continuous-time system. Specifically, for each time point $t_k$, we identify its corresponding temporal segment indexed by $s_k=\lceil k/M\rceil$ and integrate the associated segment-conditioned adjacency matrix $A^{(s_k)}$ to propagate $\mathbf{z}(t_k)$ backward to $t_0$ via the proposed GSNF:
\begin{equation}
\{\mathbf{z}^k_0\}_{k=1}^{L}
=
\left\{
F\!\left(\mathbf{z}(t_k), t_k, t_0, A^{(s_k)} \right)
\right\}_{k=1}^{L}.
\end{equation}

The segment-conditioned adjacency matrices $\{A^{(s)}\}_{s=1}^{C}$ are introduced only during inference to improve robustness under irregular missing observations.
Each IVP-based estimate $\mathbf{z}_0^k$ parameterizes a diagonal Gaussian variational component
$q_\phi(\mathbf{z}_0^k \mid X) = \mathcal{N}(\mu_{\mathbf{z}_0^k}, \sigma_{\mathbf{z}_0^k})$, where $\mu_{\mathbf{z}_0^k} = h_z(\mathbf{z}_0^k)$ and $\sigma_{\mathbf{z}_0^k} = \text{Softplus}(h_z(\mathbf{z}_0^k))$, with $h_z(\cdot)$ a shared feed-forward network. Finally, the posterior distribution over the initial latent state $\mathbf{z}_0$ is obtained from $\{\mathbf{z}^k_0\}_{k=1}^{L}$ of the sequence $X$:
\begin{equation}
\label{eq:posterior_z}
q_\phi(\mathbf{z}_0\mid X)=\frac{1}{L}\sum_{i=1}^{L} q_\phi(\mathbf{z}_0^i\mid X).
\end{equation}

%For each time series $X_n = \{(\mathbf{x}_i, t_i)\}_{i=1}^{L_n}$ with its corresponding mask $\{m\}^{L_n}_i$ that indicate which variables are observed and which are not at time $t_i$. We  concatenated $\mathbf{x}_i$ and $\mathbf{m}_i$, and process the result through  neural networks to obtain latent states $z_i=\text{Enc}(x_i|m_i)$ at each time points.

% Then, during the inference process, we incorporate the proposed GSNF as an Initial Value Problem (IVP) solver~\cite{xiao2024ivp} within the latent sequence modeling framework. Each latent state $z_i$ at timestamp $t_i$ together with its adjacency matrix $A_i$ derived in Eq. \eqref{eq:adj} serves as the initial condition and graph structure for the flow, yielding a set of intermediate latent states $z_0^i$. 

% For the evolution of the latent state $z$, during the inference process, we use each window-specific adjacency matrix $A_i$ and latent state $z_i$ as the initial condition and graph structure for Graph Neural Flow to solve for a set of intermediate latent states $\{z_0^{i=1}\}^{L}$: 

% \begin{equation}
% \{z^i_0\}_{i=1}^{L}=\{F\left(z(t_i),t_i,t_0, A\right)\}_{i=1}^{L}.
% \end{equation}

% Finally, a mean-weighted posterior distribution is computed over $z_0$ from $\{z_0^i\}$ of the sequence $X_n$:

% \begin{equation}
% q_\phi(z_0\mid X_n)=\frac{1}{L_n}\sum_{i=1}^{L_n} q_\phi(z_0^i\mid X_n),
% \label{eq:z0}
% \end{equation}
% and sample the initial latent representation $z_0$ from this distribution.

\subsection{Latent Trajectory Generation}
\label{sec:self_supervised}
With the initial latent state $\mathbf{z}_0$ sampled from \eqref{eq:posterior_z} and the interaction graph $A$ sampled from \eqref{eq:posterior_A}, latent trajectories 
are approximated using the GSNF as

\begin{equation}
\{\mathbf{z}(t_i)\}_{i=1}^{L}=\{F\left(\mathbf{z}_0, 0,t_i,A\right)\}_{i=1}^{L}.
\label{eq: z-forward}
\end{equation}

However, one-step generation applies the learned interactions only once, offering limited supervision for refining the continuous dynamics and variable dependencies encoded in the flow and the interaction graph. 
Therefore, we introduce self-supervised regularization by generating two complementary auxiliary trajectories and enforcing trajectory-level constraints, as detailed below.

\paragraph{Interaction-Aware Trajectory Generation (ITG).}
% \begin{figure}[t]
% \centering
% \includegraphics[width=0.9\linewidth,]{figure/ITG.png}
% \caption{Interaction-Aware Trajectory Generation.}
% \label{fig:itg}
% \end{figure}

% \begin{figure}[t]
% \centering
% \includegraphics[width=0.9\linewidth,]{figure/traj_4.0.png}
% \caption{Interaction-aware trajectories generation via re-initialization. A new initial condition is sampled from the original trajectory, and the trajectory is re-evolved under the same dynamics. This process implicitly captures the interactions among variables, illustrated here with the example $z(t)=\exp\big[(I-A^\top)(t-t_0)]z(t_0)$.}
% \label{fig:traj_evolution}
% \end{figure}

Following the observations in~\cite{mercatali2024graph}, non-interacting components exhibit invariant behavior under perturbations, whereas components involved in graph-induced interactions show sensitivity to re-initialization. Building on this insight, GSNF constructs interaction-aware auxiliary trajectories by re-initializing the learned flow approximation along the original latent trajectory. 

Specifically, for a re-initialization time $t_0^* \in (t_0, t_L)$, we define the corresponding state
$
\mathbf{z}_0^* = F(\mathbf{z}(t_0), t_0, t_0^*, A)$,
and generate an auxiliary trajectory using the same parameterized GSNF, with  $\mathbf z_0^\ast$ serving as the re-initial condition:
\begin{equation}
\mathbf{z}^\ast(t)=F(\mathbf{z}^\ast_0, t^\ast_0, t, A), \quad t \ge t_0^*.
\end{equation}

The divergence between the original trajectory $\mathbf{z}(t)$ and the re-initialized auxiliary trajectory $\mathbf{z}^*(t)$ reflects the sensitivity of the learned flow
to interaction-induced perturbations. To explicitly encourage such sensitivity, we
introduce a margin-based trajectory-level regularizer:
% \begin{equation}
%     \mathcal{L}_{\text{ITG}} =\frac{1}{L} \max(\sum_{i=1}^L \|\mathbf{z}(t_i)-\mathbf{z}^\ast(t_i)\|^2,  \delta^2),
% \end{equation}

% \begin{equation}
% \mathcal{L}_{\text{ITG}}
% =
% \frac{1}{L}
% \max\!\Big(
% \delta^2
% -
% \sum_{i=1}^L \|\mathbf{z}(t_i)-\mathbf{z}^\ast(t_i)\|^2,
% \; 0
% \Big),
% \end{equation}

\begin{equation}
\label{eq:ITG}
\mathcal{L}_{\mathrm{ITG}}
=
\frac{1}{L-k_0^*+1}
\max\!\Bigg(
\delta^2
-
\sum_{i=k_0^*}^{L}
\bigl\|\mathbf z(t_i)-\mathbf z^\ast(t_i)\bigr\|^2,
\; 0
\Bigg),
\end{equation}
where the summation starts from the re-initialization index $k_0^*$ for $t_i \ge t_0^*$, and $\delta>0$ sets a separation margin.

In this work, the separation margin $\delta$ can be treated as a fixed
hyperparameter, or alternatively chosen based on a theoretical lower-bound
analysis provided in Theorem~\ref{thm:itg_lower_bound}, without manual tuning.

\begin{theorem}[A Data-Dependent Lower Bound for the ITG Separation Margin]
\label{thm:itg_lower_bound}
Assume that the graph-conditioned interaction module in
Eq.~\eqref{eq:gsnf} admits the form
$
\mathrm{GCN}(\mathbf{z}(t_0), t_0, t, A)=\mathcal A \mathbf z(t_0) W,
$
where $\mathcal A$ is the normalized adjacency matrix and $W$ is a trainable
weight matrix. Let $\mathbf z(t)$ and $\mathbf z^*(t)$ denote the original and re-initialized
latent trajectories, respectively.
Then, for all discrete time points $t_i>t_0^*$, the cumulative trajectory
divergence satisfies
\begin{equation}
\label{eq:lower_bound}
\sum_{i=k_0^*}^{L}\|\mathbf z^*(t_i)-\mathbf z(t_i)\|
\;\ge\;
\max\!\Bigl\{0,\,(L-k_0^*+1)\bigl(\eta-\Delta_{\mathrm{in}}\bigr)\Bigr\},
\end{equation}
where 
$\Delta_{\mathrm{in}}=\|\mathbf z_0^*-\mathbf z_0\|, ~
\eta
=
\sigma_{\min}(\mathcal A)\,
\sigma_{\min}(W)\,
\Delta_{\mathrm{in}}
,$
and $k_0^*$ denotes the index of the first time point $t_i \ge t_0^*$, $\sigma_{\min}(\cdot)$ denotes the smallest singular value of a matrix.

\end{theorem}

This bound characterizes the minimal divergence induced by the interaction structure when re-initialization perturbations are amplified by graph-conditioned dynamics.

\paragraph{Reverse-time Trajectory Generation (RTG).}
% \begin{figure}[t]
% \centering
% \includegraphics[width=0.9\linewidth,]{figure/RTG.png}
% \caption{Reverse-time Trajectory Generation.}
% \label{fig:rtg}
% \end{figure}

Leveraging the invertibility of GSNF guaranteed by Theorem~\ref{thm:inveritbility},  we generate an additional self-supervised signal by reversing the learned latent trajectory. Starting from the final latent state $\mathbf{z}(t_L)$ obtained in Eq.\eqref{eq: z-forward},  we apply the learned flow in reverse time to obtain a reversed latent
trajectory $\{\tilde{\mathbf{z}}(t_k)\}_{k=1}^L$:
\begin{equation}
    \{\tilde{\mathbf{z}}(t_i)\}_{i=1}^L=\{F(\mathbf{z}(t_L), t_L, t_i, A)\}^L_{i=1}.
\end{equation}

The reversed trajectory is decoded and compared to the input sequence using the reconstruction loss:

\begin{equation}
\label{eq:RTG}
\mathcal{L}_{\text{RTG}} = \frac{1}{L} \sum_{i=1}^{L}\left\| f_\mathrm{dec}(\tilde{\mathbf{z}}(t_i)) - \mathbf{x}(t_i) \right\|^2.
\end{equation}
This reverse-time regularization encourages forward–backward consistency of the learned flow.
%This design provides a self-supervised regularization signal that encourages forward–backward consistency  of the learned flow and supports accurate reconstruction of the input sequence.

\subsection{Training}
Our model is a VAE-based framework augmented with a generative flow to capture latent interactions and ensure temporally consistent dynamics. The overall objective combines variational inference with supervised classification, and trajectory-level self-supervision:
\begin{equation}
\begin{aligned}
\mathcal{L}
=\, & \mathcal{L}_{\text{VAE}} +\, \alpha \mathcal{L}_{\text{CE}}+\, \beta\,\mathcal{L}_{\text{ITG}}+\, \gamma\,\mathcal{L}_{\text{RTG}},
\label{eq:loss}
\end{aligned}
\end{equation}
where $L_\text{VAE}$ is the evidence lower bound (ELBO)~\cite{kingma2013auto},
\begin{equation}
\begin{aligned}
\mathcal{L}_{\text{VAE}}(\phi, \theta) &= 
\mathbb{E}_{\mathbf{z}_0 \sim q_\phi(\mathbf{z}_0| X)}\big[\log p_\theta(X|\mathbf{z}_0)\big] \\
&\quad - \frac{1}{L}\sum_{i=1}^L D_{\text{KL}}\big(q_\phi(\mathbf{z}_0^i|X)\,\|\,p(\mathbf{z}_0)\big),
\end{aligned}
\label{eq:vae_loss}
\end{equation}
and $\mathcal{L}_{\text{CE}}$ is the cross-entropy loss for classification:

\begin{equation}
    \mathcal{L}_{\text{CE}} = -\sum_{y} p(y) \log p_{\lambda}(y \mid \mathbf{z}_{0}).
\end{equation}

The trajectory-level losses $\mathcal{L}_{\text{ITG}}$ and $\mathcal{L}_{\text{RTG}}$ act as regularizers complementing the main objective. $\mathcal{L}_{\text{ITG}}$ enforces divergence between original and re-initialized trajectories to capture interactions, while $\mathcal{L}_{\text{RTG}}$ promotes temporal consistency via reverse-time reconstruction. The relative weights $\alpha$, $\beta$, and $\gamma$ are reported in Appendix~\ref{app:hyperparams}.

\paragraph{ITG Training Variants.} 
For the ITG loss, we consider two variants for defining the minimum separation margin: 
(i) a fixed hyperparameter $\delta$; 
(ii) a theoretically derived margin $\delta_\mathrm{lb}$ set according to the lower bound in Theorem~\ref{thm:itg_lower_bound}. 
The latter requires no manual tuning, as the margin is automatically adapted during training along with the evolving model parameters.

\section{Experiments}
%In this section, we outline the experimental setup, introduce the datasets used for comparison, and present the experimental results with analysis.

\subsection{Datasets and Baselines}
We evaluated GSNF on five datasets for irregular time series classification: PhysioNet12~\cite{silva2012predicting}, P12~\cite{goldberger2000physiobank}, P19~\cite{reyna2020early}, MIMIC-IV~\cite{johnson2020mimic}, eICU~\cite{pollard2018eicu}. Detailed experimental settings are provided in Appendix~\ref{app:experiment}.

We compare with representative baselines from three major categories:
(i) continuous-time models, including GRU-D~\cite{che2018recurrent}, ODE-RNN~\cite{rubanova2019latent}, NeuralFlow~\cite{bilovs2021neural}, IVP-VAE~\cite{xiao2024ivp}, DualDynamics~\cite{oh2025dualdynamics} and FlowPath~\cite{oh2025flowpath};
(ii) graph-based models, including Raindrop~\cite{zhang2021graph}, GNeuralFlow~\cite{mercatali2024graph} and Hi-Patch~\cite{luohi};
(iii) other strong baselines, including mTAN~\cite{shukla2021multi}, ViTST~\cite{li2023time}, Warpformer~\cite{zhang2023warpformer}, and TimeCHEAT~\cite{liu2025timecheat}.

\subsection{Overall Comparisons}

% Table 2
\begin{table*}[t]
\centering
\begin{threeparttable}
\small
\setlength{\tabcolsep}{3.5pt}
\renewcommand{\arraystretch}{1.15}
\begin{tabular}{l cc cc cc cc cc cc}
\toprule
& \multicolumn{2}{c}{Physionet12}
& \multicolumn{2}{c}{P12}
& \multicolumn{2}{c}{P19}
& \multicolumn{2}{c}{MIMIC-IV} 
& \multicolumn{2}{c}{eICU}
\\ % <-- replace with your dataset name
\cmidrule(lr){2-3}\cmidrule(lr){4-5}\cmidrule(lr){6-7}\cmidrule(lr){8-9}\cmidrule(lr){10-11}
Method & AUROC & AUPRC & AUROC & AUPRC & AUROC & AUPRC & AUROC & AUPRC &AUROC & AUPRC \\
\midrule
GRU-D           & 79.1$\pm$6.9 & 42.7$\pm$7.2  & 81.9$\pm$2.1 & 46.1$\pm$4.7  &83.7$\pm$1.5 &46.9$\pm$2.1 & 82.2$\pm$1.8 & 48.3$\pm$2.1 & 84.6$\pm$1.5 & 49.6$\pm$2.4\\
ODE-RNN         & 80.8$\pm$0.4 & 33.7$\pm$4.1  & 81.5$\pm$0.5 & 42.3$\pm$0.7  &81.4$\pm$0.8 &45.7$\pm$1.0 & 81.0$\pm$0.6 & 47.3$\pm$0.7 & 83.3$\pm$0.9 & 50.3$\pm$1.5\\
NeuralFlow      & 80.9$\pm$0.1 & 51.5$\pm$1.8  & 81.3$\pm$0.1 & 51.5$\pm$2.0   &84.1$\pm$0.7&52.1$\pm$3.3 & 79.7$\pm$0.4 & 51.5$\pm$1.2 & 83.1$\pm$0.8 & 50.7$\pm$1.8\\
IVP-VAE         & 81.1$\pm$2.3 & 46.2$\pm$2.3  & 81.8$\pm$0.2 & 53.5$\pm$1.5  &85.6$\pm$1.2 &53.7$\pm$2.7 & 81.8$\pm$0.5 & 52.7$\pm$1.4 & 83.6$\pm$1.7 & 53.2$\pm$2.9\\
DualDynamics    & 86.1$\pm$0.1 & 55.3$\pm$1.5  & 86.5$\pm$0.2 & 53.2$\pm$0.8 &90.7$\pm$0.4& 56.7$\pm$1.4  & 84.4$\pm$0.4 & 54.7$\pm$0.9 & 84.9$\pm$1.1 & 54.2$\pm$1.5\\
FlowPath    & 85.3$\pm$0.2 & 55.3$\pm$1.7  &\underline{86.6$\pm$0.3} &54.3$\pm$0.7&88.4$\pm$0.7& 56.4$\pm$1.3 &84.4$\pm$0.7 & 53.2$\pm$0.8  & 85.1$\pm$1.3 & 53.8$\pm$1.4\\
\midrule
Raindrop        & 81.2$\pm$0.9 & 37.3$\pm$2.0  & 82.8$\pm$1.7 & 39.4$\pm$2.4  &87.0$\pm$2.3&51.8$\pm$5.5 & 79.8$\pm$1.3 & 35.2$\pm$1.1 & 84.6$\pm$2.1 & 55.1$\pm$2.7\\
GNeuralFlow     & 84.5$\pm$0.8 & 53.7$\pm$2.4  & 85.5$\pm$0.8 & \underline{55.5$\pm$1.8}  &88.4$\pm$0.7&56.4$\pm$1.3 & 83.6$\pm$0.7 & 53.1$\pm$0.8& 85.1$\pm$0.7 & 54.5$\pm$2.6\\
Hi-Patch       & 86.4$\pm$0.9&  56.5$\pm$1.4  & 86.5$\pm$0.6&  53.3$\pm$0.9 & 88.4$\pm$1.1 & 56.2$\pm$3.3 &  84.9$\pm$0.2 &  54.2$\pm$1.0 &  84.5$\pm$0.7& 55.3$\pm$1.3\\
\midrule
mTAN            & 85.8$\pm$1.3 & 50.4$\pm$1.0  & 84.2$\pm$0.8 & 52.5$\pm$1.3 &84.0$\pm$1.3&50.6$\pm$2.0 & 83.8$\pm$0.3 & 46.6$\pm$0.5 & 80.9$\pm$2.4 & 48.1$\pm$3.2\\
ViTST           & 81.3$\pm$1.9 & 37.4$\pm$2.9  & 86.3$\pm$0.1 & 50.8$\pm$1.5 &88.3$\pm$2.0&53.1$\pm$3.4 & 81.8$\pm$0.3 & 39.6$\pm$1.3 & 80.5$\pm$1.3 & 47.7$\pm$2.4\\
Warpformer      & 83.4$\pm$0.7 & 43.5$\pm$2.3  & 85.4$\pm$0.5 & 50.4$\pm$1.5  &87.2$\pm$1.8&52.7$\pm$2.4 & 84.6$\pm$0.3 & 46.6$\pm$0.9 & 84.8$\pm$0.4 & 53.7$\pm$1.1\\
TimeCHEAT       & 84.5$\pm$0.7 & 46.3$\pm$1.5  & 84.6$\pm$0.7 & 48.2$\pm$1.9  &89.1$\pm$1.4&55.2$\pm$1.9 & 83.1$\pm$0.3 & 51.2$\pm$0.9 & 83.9$\pm$1.4 & 54.4$\pm$1.5\\

\midrule
GSNF($\delta$)            & \underline{86.4$\pm$1.0} & \underline{56.8$\pm$1.0}  & 86.6$\pm$0.1 & 54.6$\pm$1.4   &\underline{90.7$\pm$0.6} &\underline{57.4$\pm$2.3} & \underline{84.9$\pm$0.5} & \underline{55.3$\pm$1.3} & \textbf{85.2$\pm$1.4} & \underline{56.1$\pm$1.0}\\

GSNF($\delta_\mathrm{lb}$)   & \textbf{86.7$\pm$0.6} & \textbf{56.9$\pm$0.7}  & \textbf{87.5$\pm$0.3} & \textbf{56.3$\pm$0.4}  &\textbf{91.2$\pm$0.5} &\textbf{57.9$\pm$2.4}  & \textbf{85.3$\pm$0.5} & \textbf{55.5$\pm$1.5} & \underline{85.2$\pm$0.9} & \textbf{56.1$\pm$1.1}\\
\bottomrule
\end{tabular}
\end{threeparttable}

\caption{Comparison of GSNF and baseline methods on classification tasks across five datasets. GSNF variants achieve the top performance across datasets, with the theoretically derived variant GSNF($\delta_{\mathrm{lb}}$) performing better in most cases.
}
\label{tab:main_result}
\end{table*}

\begin{figure}[t]
\centering
\includegraphics[width=\linewidth,]{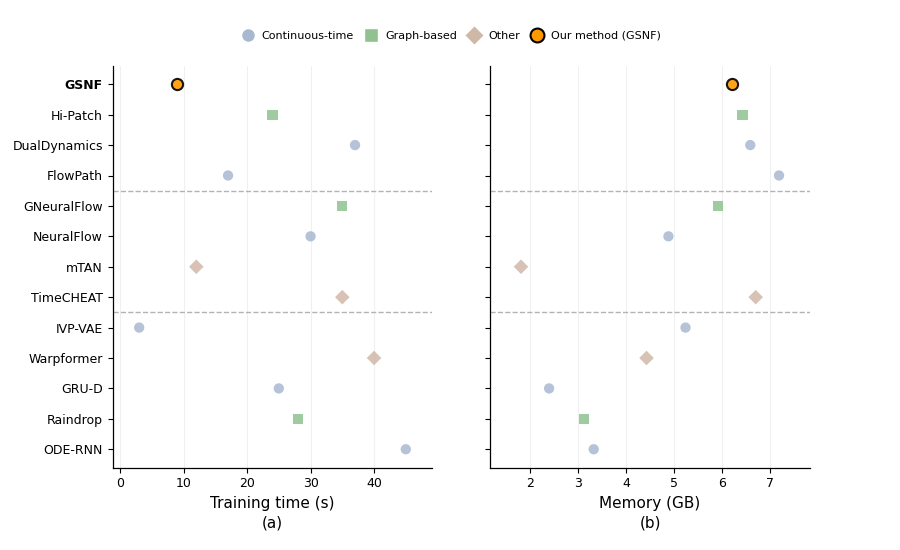}
\caption{Computational efficiency on the PhysioNet12 dataset.
(a) Training time and (b) peak GPU memory usage.
Methods are ordered from top to bottom by decreasing AUPRC, with dashed lines indicating three performance tiers. GSNF achieves the highest AUPRC and attains the lowest training time and memory usage within the top performance tier.
Different marker shapes denote different categories of methods.}
\label{fig:efficiency}
\end{figure}

Overall results are summarized in Table \ref{tab:main_result}. Across all five datasets, the top two models in terms of AUROC and AUPRC are both variants of our method GSNF, indicating stable performance across datasets with diverse missing rates, label imbalance, and feature dimensionality. 

Notably, GSNF achieves consistently strong improvements in AUPRC, which is more indicative than AUROC in highly imbalanced tasks. The largest gain is observed on P19, where GSNF outperforms the second-best method by 1.2 AUPRC points. Since P19 has the lowest positive rate among all evaluated datasets, this result suggests that incorporating inter-variable interactions is especially beneficial under severe class imbalance, enabling the model to better identify rare but clinically critical positive cases.

In addition,  the variant of GSNF with the theoretically derived margin, GSNF($\delta_\mathrm{lb}$), outperforms the best manually tuned variant GSNF($\delta$) on four out of five datasets (PhysioNet12, P12, P19, and MIMIC-IV), and achieves comparable performance on eICU. This observation indicates that the theoretically derived margin $\delta_\mathrm{lb}$ provides a principled alternative to dataset-specific hyperparameter tuning and yields robust performance in practice. 
As a result, $\delta_\mathrm{lb}$ offers a reliable and practical mechanism for regulating trajectory separation without incurring additional tuning overhead.

Fig.~\ref{fig:efficiency} further examines computational efficiency. Within the top AUPRC tier, GSNF achieves the lowest training time. This efficiency is largely attributed to the IVP-VAE backbone, which exhibits the lowest training time among all the baselines. Notably, despite being the fastest baseline, IVP-VAE lies in the third performance tier; GSNF incorporating interaction modeling resolves this limitation by elevating GSNF to the top AUPRC tier, while retaining the second-lowest training time overall. Meanwhile, peak GPU memory generally increases with higher AUPRC across methods. For clarity, only one GSNF point is shown in Fig.~\ref{fig:efficiency}, as the two GSNF variants have comparable training time and memory usage due to the negligible overhead of the lower-bound computation. Detailed numerical results are provided in Appendix~\ref{app:performance_analysis}.

\subsection{Ablation Study}
\begin{table}[t]
\centering
\begin{threeparttable}
\footnotesize % 比\small还小
\setlength{\tabcolsep}{2pt} 
\renewcommand{\arraystretch}{1.15}
\resizebox{1.0\columnwidth}{!}{%
\begin{tabular}{lcccccccccc}
\toprule
 & \multicolumn{2}{c}{Physionet12} & \multicolumn{2}{c}{P12} & \multicolumn{2}{c}{P19} &\multicolumn{2}{c}{MIMIC-IV} & \multicolumn{2}{c}{eICU}\\
\cmidrule(lr){2-3} \cmidrule(lr){4-5} \cmidrule(lr){6-7} \cmidrule(lr){8-9} \cmidrule(lr){10-11}
Method & AUROC & AUPRC & AUROC & AUPRC & AUROC & AUPRC & AUROC & AUPRC & AUROC & AUPRC\\
\midrule
GSNF                   & \textbf{86.7} & \textbf{56.9}  & \textbf{87.5} & \textbf{56.3}  & \textbf{91.2} & \textbf{57.9}& \textbf{85.3} & \textbf{55.5} & \textbf{85.2} & \textbf{56.1}\\
-w/o ITG               & 84.4 & 51.3  & 84.7 & 51.2 & 87.9 & 54.4 & 85.1 & 52.3 & 84.2 & 52.1\\
-w/o RTG               & 85.5 & 53.7& 85.5 & 53.0 & 89.2 & 55.7& 84.5 & 53.9  & 84.0 & 53.7 \\
\hspace{1em}-w/o RTG\&ITG   & 84.1 & 48.9  & 83.9 & 50.1 & 86.1 & 53.8& 83.1 & 51.9 & 83.5 & 51.4\\
-w/o graph             & 82.7 & 46.4 & 82.7 & 48.2 & 84.8 & 52.9 & 81.4 & 50.2  & 81.6 & 49.8\\
\bottomrule
\end{tabular}
}
\end{threeparttable}
\caption{Ablation study. The -w/o graph variant yields the largest drop, especially in AUPRC, followed by -w/o ITG+RTG; -w/o ITG degrades performance more than -w/o RTG.}
\label{tab:ablation1}
\end{table}

\begin{figure}[t]
\centering
\begin{subfigure}[t]{0.45\columnwidth}
    \centering
    \includegraphics[width=\linewidth]{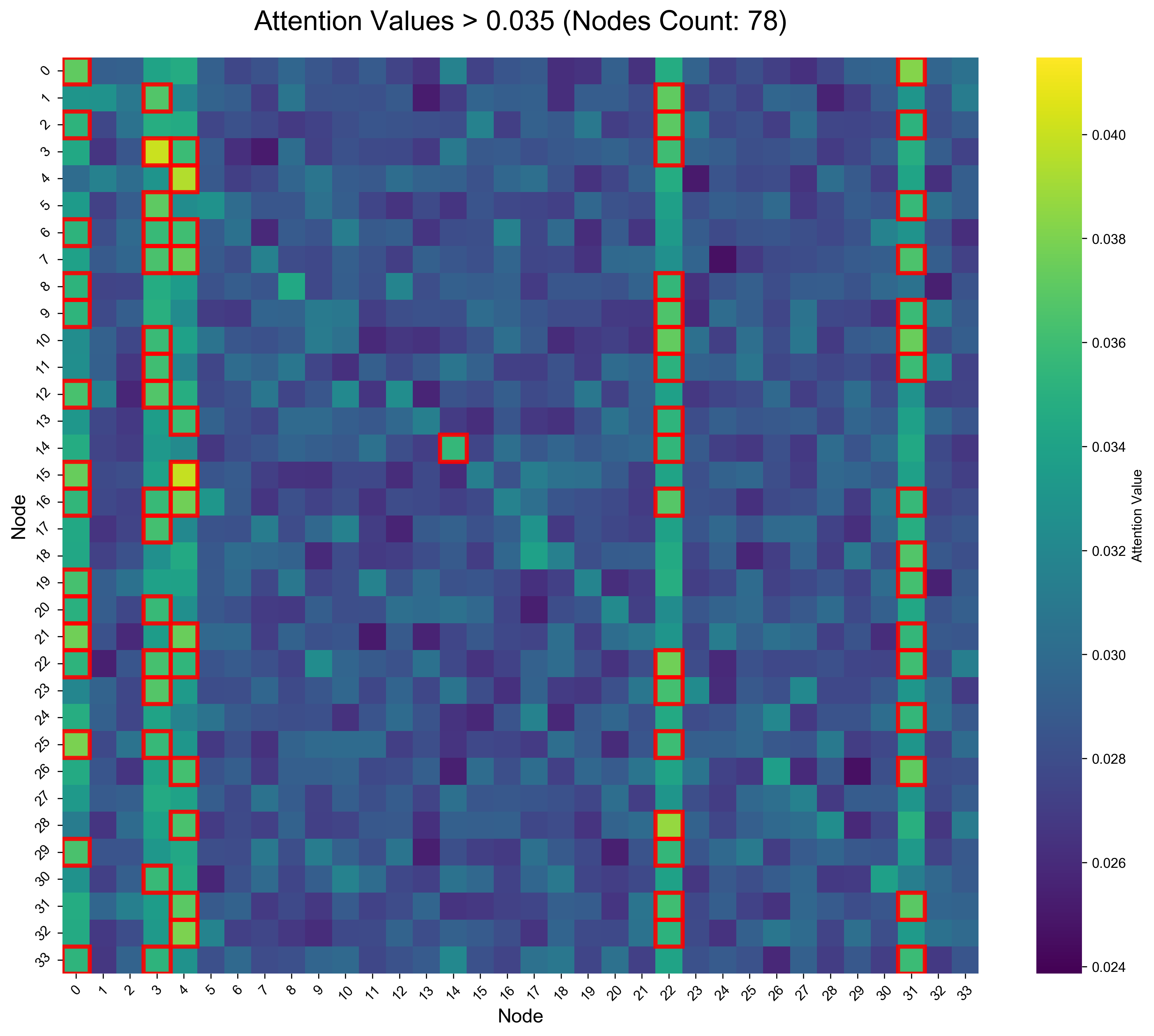}
    \caption{GSNF($\delta_{lb}$)}
    \label{fig:ablantion_GSNF_lb}
\end{subfigure}
% \hfill
% \begin{subfigure}[t]{0.45\columnwidth}
%     \centering
%     \includegraphics[width=\linewidth]{figure/attention_GSNF.png}
%     \caption{GSNF($\delta$)}
%     \label{fig:ablantion_GSNF}
% \end{subfigure}
\hfill
\begin{subfigure}[t]{0.45\columnwidth}
    \centering
    \includegraphics[width=\linewidth]{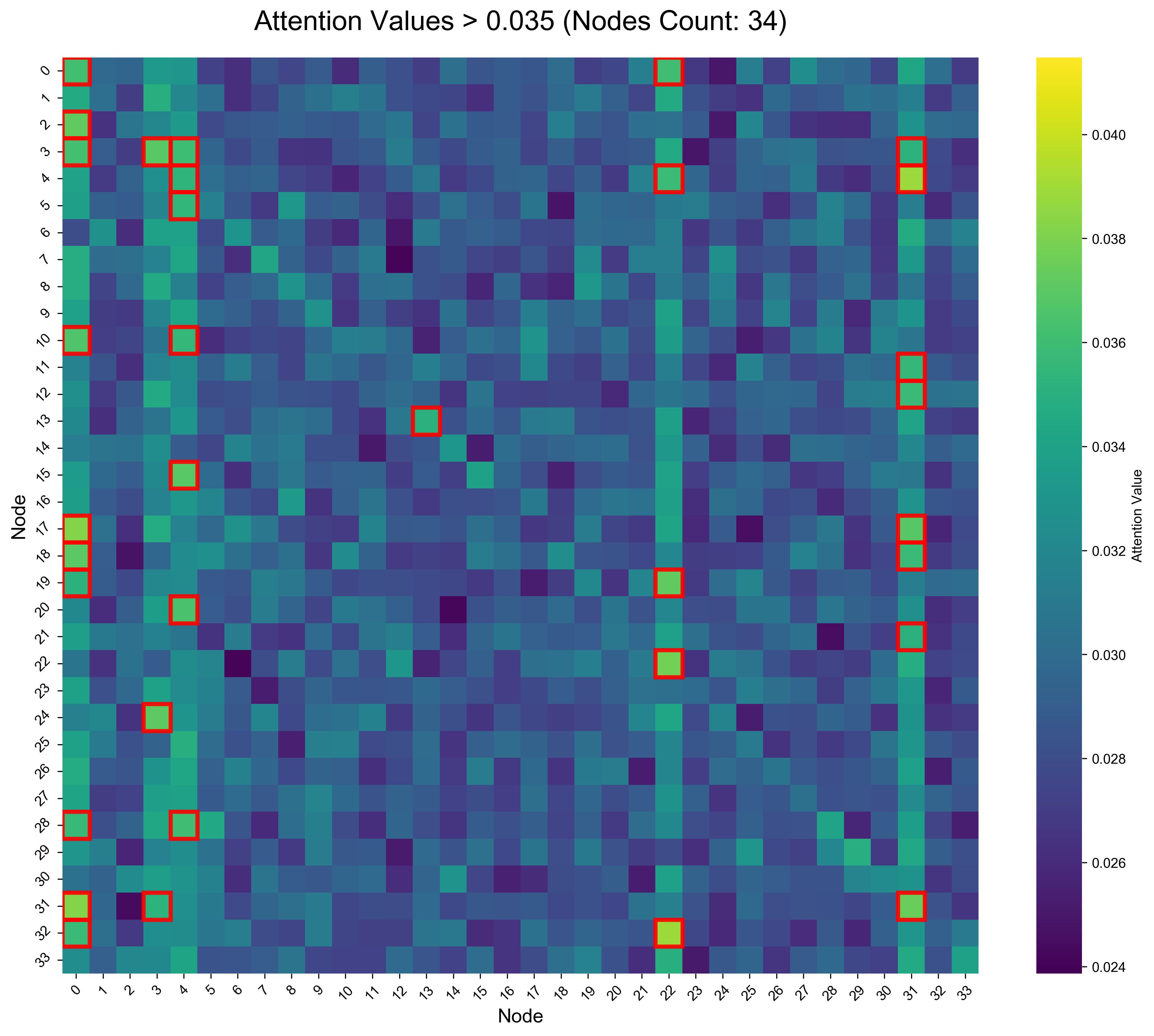}
    \caption{-w/o ITG}
    \label{fig:ablantion_ITG}
\end{subfigure}
\caption{ Interaction graphs with and without ITG  on the P19 dataset.
Attention weights are normalized by softmax, and edges highlighted by red boxes denote attention values above 0.035. With ITG, the interaction graph contains substantially more salient connections (78 edges) than without ITG (34 edges), with the additional connections primarily concentrated around key clinical variables. (0: HR, 3: SBP, 4: MAP, 22: lactate, 31: WBC).
%(0: Heart rate, 3: Systolic blood pressure, 4: Mean arterial pressure, 22: Lactic acid, 31: White blood cell count).
}
\label{fig:graph_ablantion}
\end{figure}

\begin{figure}[t]
\centering
\begin{subfigure}[t]{0.45\columnwidth}
    \centering
    \includegraphics[width=\linewidth]{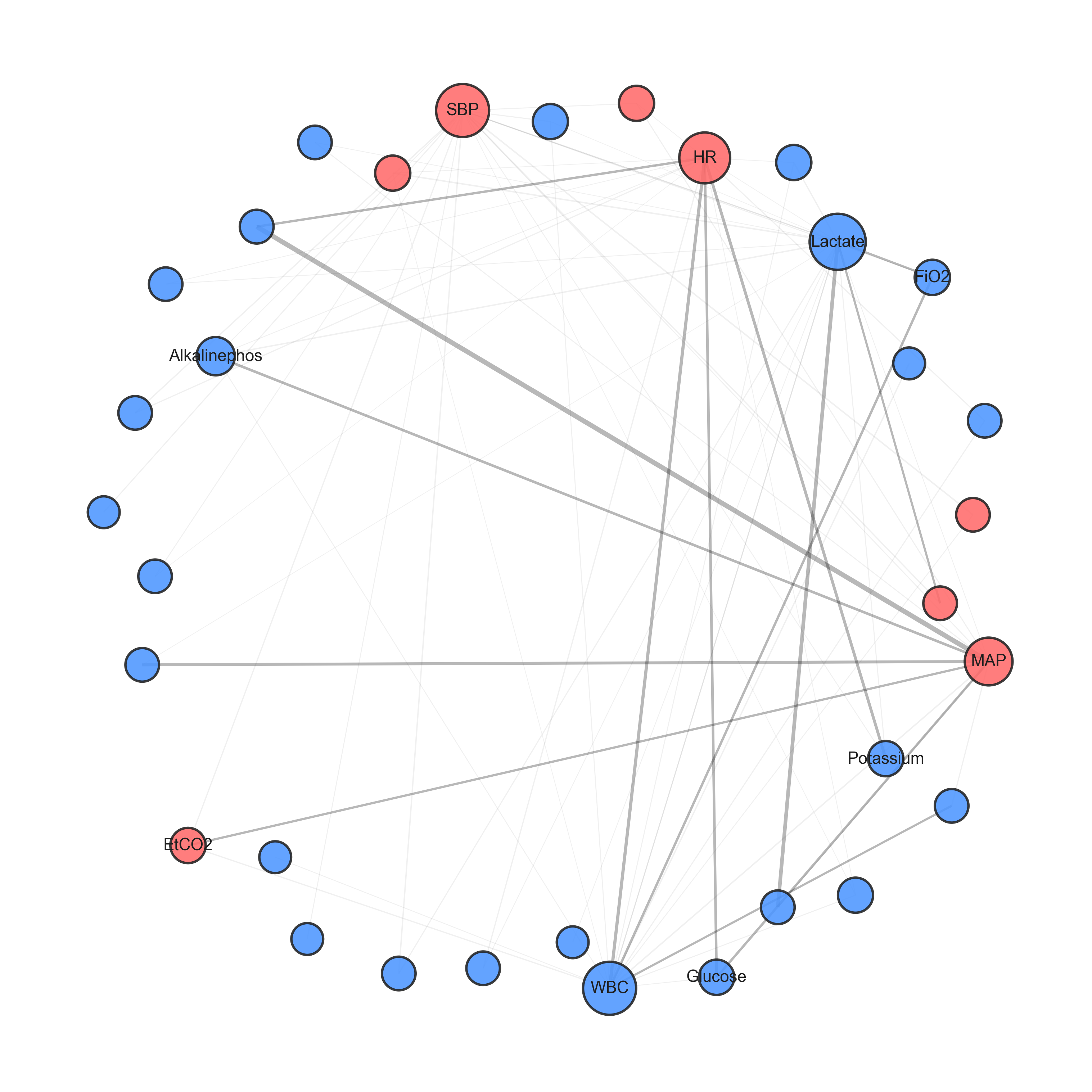}
    \caption{Graph of GSNF($\delta_{lb}$)}
    \label{fig:case_graph}
\end{subfigure}
% \hfill
% \begin{subfigure}[t]{0.45\columnwidth}
%     \centering
%     \includegraphics[width=\linewidth]{figure/attention_GSNF.png}
%     \caption{GSNF($\delta$)}
%     \label{fig:ablantion_GSNF}
% \end{subfigure}
\hfill
\begin{subfigure}[t]{0.45\columnwidth}
    \centering
    \includegraphics[width=\linewidth]{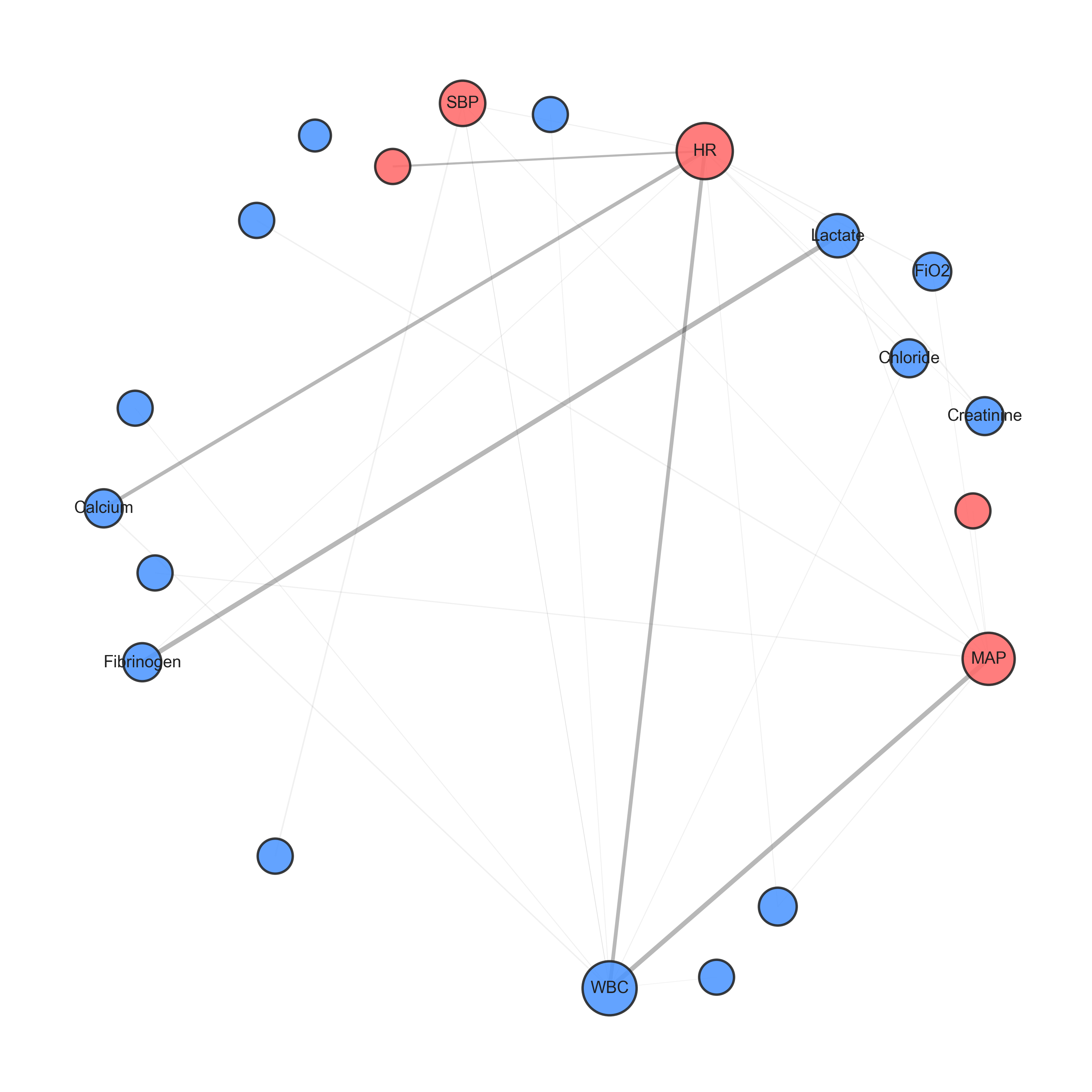}
    \caption{-w/o ITG}
    \label{fig:case_graph_woITG}
\end{subfigure}
\caption{Visualization of learned interaction graph  on P19. Node size reflects weighted degree, node color denotes variable type
(red: vital signs; blue: laboratory variables; left: 8/25; right: 5/15). Only the top 20\% strongest interactions are shown, with labels displayed for the most
connected nodes, highlighting the structural differences induced by ITG.}
\label{fig:graph_ablantion2}
\end{figure}

% \begin{figure}[t]
% \centering
% \begin{subfigure}[t]{0.45\columnwidth}
%     \centering
%     \includegraphics[width=\linewidth]{figure/Adj_1.png}
%     \caption{GSNF}
%     \label{fig:graph_with_repulsion}
% \end{subfigure}
% \hfill
% \begin{subfigure}[t]{0.45\columnwidth}
%     \centering
%     \includegraphics[width=\linewidth]{figure/Adj_2.png}
%     \caption{-w/o ITG}
%     \label{fig:graph_without_repulsion}
% \end{subfigure}
% \caption{Graph structures with and without ITG. 
% The highlighted nodes (11: Glasgow Coma Score, 36: White Blood Cell Count) are shown to have stronger and more consistent connections when ITG is applied in GSNF. These nodes correspond to clinically important vital signs for mortality prediction, indicating that ITG helps capture critical variable interactions.}
% \label{fig:graph_comparison}
% \end{figure}

We construct four ablation variants of GSNF by removing ITG, RTG, both ITG and RTG, or the interaction graph. Table~\ref{tab:ablation1} presents 
the corresponding ablation results across five datasets, showing consistent performance degradation when any component is removed.

Among all components, removing the interaction graph results in the largest performance degradation across datasets, especially in AUPRC, highlighting the necessity of explicitly modeling inter-variable interactions. Removing both ITG and RTG leads to the second-largest performance drop, indicating that these two trajectory-level supervision mechanisms jointly contribute to stabilizing and refining interaction learning. When considered individually, removing ITG causes a larger performance decrease than removing RTG.

Beyond numerical performance, Fig.~\ref{fig:graph_ablantion} 
provides a qualitative view of how ITG influences interaction learning. Attention weights are normalized by softmax, and red boxes indicate edges whose attention values exceed 0.035.  In both GSNF and its w/o ITG ablation, high-attention connections are primarily  associated with a small set of clinically important variables (0: HR, heart rate; 3: SBP, systolic blood pressure; 4: MAP, mean arterial pressure;
22: lactate; 31: WBC, white blood cell count),  forming visually prominent vertical regions. When ITG is enabled, the number of edges highlighted by red boxes increases from 34 to 78, with the additional connections remaining concentrated around these variables. This suggests that ITG
reinforces coordinated interactions rather than introducing diffuse or noisy edges.

Further, Fig.~\ref{fig:graph_ablantion2} 
complements this analysis by visualizing the learned interaction graphs. The interaction graph learned by GSNF exhibits stronger and denser connectivity, with HR, SBP, MAP, lactate, and WBC acting as clear hub nodes around which high-weight
edges concentrate to form a coherent interaction backbone. In contrast, the w/o ITG ablation produces few strong edges, with limited
connectivity and substantially reduced structural complexity.
Most nodes remain weakly connected or isolated, leading to a sparse graph that captures
only a small subset of interactions.

% Fig.~\ref{fig:graph_ablantion2} visualizes the learned interaction graphs under two settings. Consistent with the attention weight patterns observed in Fig.~\ref{fig:graph_ablantion}, the LB graph (Right) effectively captures the centrality of clinically critical variables. Specifically, Heart Rate, SBP, MAP, Lactate, and WBC emerge as prominent central nodes with dense connections (highlighted in red for vital signs). This structural alignment with medical intuition is significant, as hemodynamic instability (reflected by SBP, MAP, HR) and inflammatory markers (Lactate, WBC) are primary drivers of patient outcomes in critical care. The pronounced centrality and rich connectivity of these nodes confirm that the data-dependent margin $\delta_{lb}$ effectively regularizes trajectory generation, revealing robust and interpretable interactions.In contrast, the w/o ITG graph (Left) exhibits a significantly sparser topology, retaining only a skeletal structure with scattered connections. The distinct evolution from this sparse baseline to the highly interconnected, clinically aligned structure in the LB setting validates that the ITG mechanism successfully strengthens interaction learning by reinforcing coordinated relationships among key physiological variables.

\subsection{Parameters Sensitivity}
We analyze the sensitivity of ITG to the re-initialization time and the tuned margin $\delta$ on PhysioNet12, as shown in Figure~\ref{fig:param_sensitivity}, and further report results using the theoretical lower-bound $\delta_{\mathrm{lb}}$ as a non-tunable reference in Figure~\ref{fig:traj_evolution2}.

\begin{figure}[t]
\centering
\begin{subfigure}[t]{0.5\columnwidth}
    \centering
    \includegraphics[width=\linewidth]{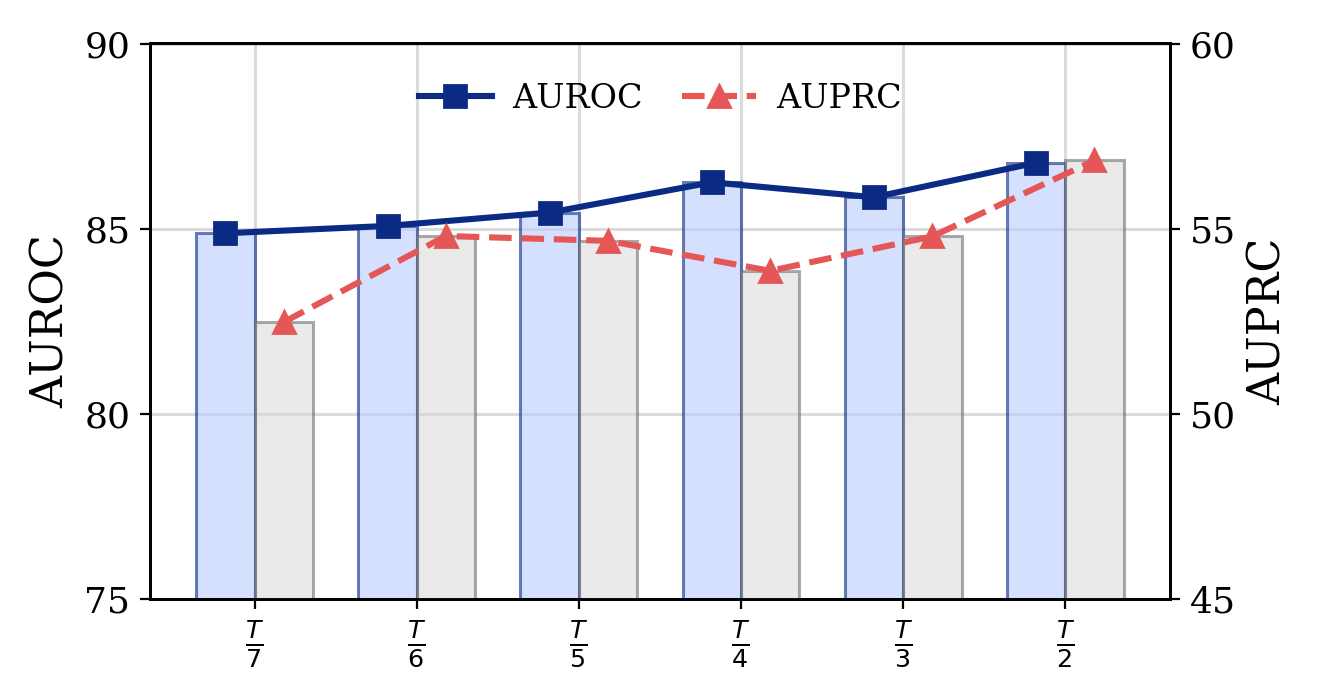}
    \caption{Re-initialization time $t_0^*$}
    \label{fig:sensitivity_T}
\end{subfigure}%
\begin{subfigure}[t]{0.5\columnwidth}
    \centering
    \includegraphics[width=\linewidth]{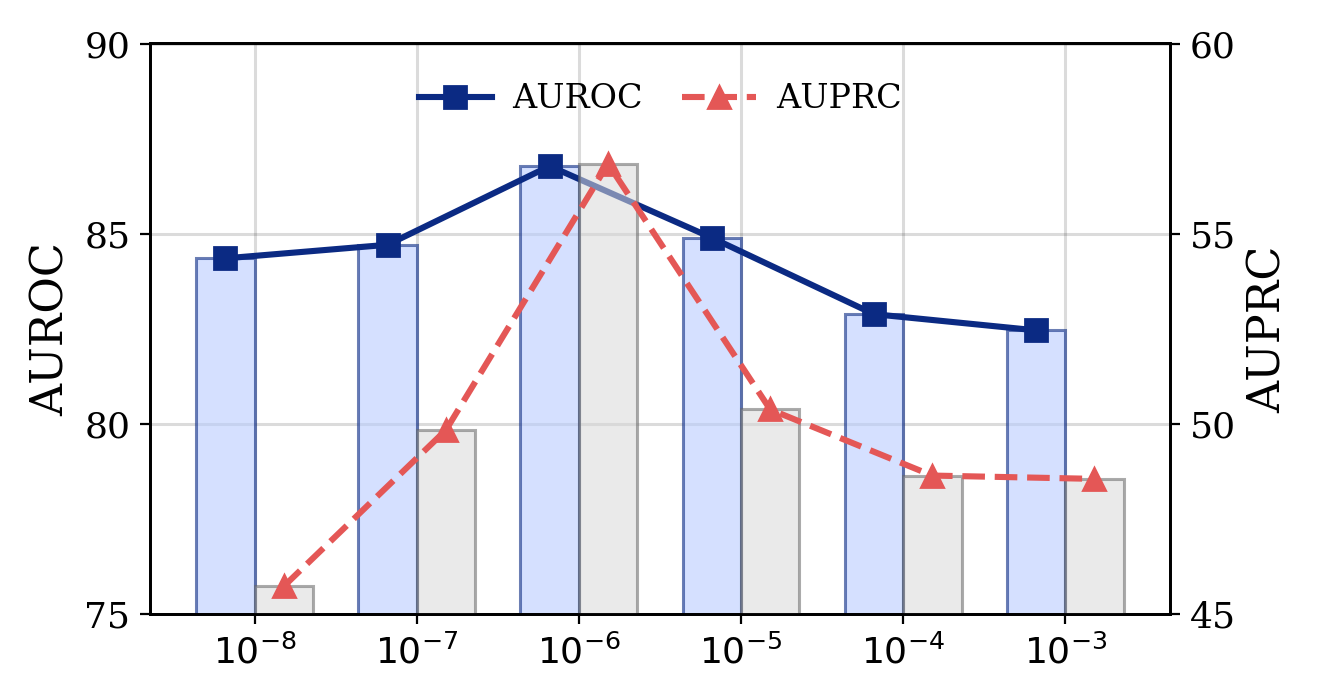}
    \caption{Fixed Separation Margin $\delta$}
    \label{fig:sensitivity_delta}
\end{subfigure}
\caption{Parameter sensitivity analysis of GSNF on PhysioNet12. Performance peaks at an intermediate $t_0^*$ and depends on $\delta$ mainly at the order-of-magnitude scale, with an optimum around $10^{-6}$.}
\label{fig:param_sensitivity}
\end{figure}

\paragraph{Re-initialization time $t_0^*$.} Figure~\ref{fig:param_sensitivity}(a) shows that GSNF achieves the best performance when the re-initialization time $t_0^*$ is set to an intermediate point around $L/2$. When $t_0^*$ is chosen too early, the re-initialized trajectories remain similar to the original ones, resulting in limited divergence and weakened ITG supervision. In contrast, setting $t_0^*$ too late reduces the effective evolution window, restricting interaction exposure and leading to degraded performance.

\paragraph{Manually Selected Separation Margin $\delta$.} 
Figure~\ref{fig:param_sensitivity}(b) shows that while performance varies across orders of magnitude of $\delta$, it remains relatively stable within each order, with the best results attained around $\delta = 10^{-6}$. This indicates that GSNF is robust to fine-grained variations of $\delta$ and does not require precise tuning. This is because overly large margins induce excessive trajectory divergence, whereas overly small margins provide insufficient separation of interaction-induced differences, reflecting a trade-off between interaction exposure and dynamical fidelity.

\paragraph{Theoretically Derived  Separation Margin $\delta_{lb}$.}
As shown in Fig.~\ref{fig:traj_evolution2}, $\delta_{lb}$ increases during the early stages of training and gradually converges to a stable value as training proceeds. This behavior reflects the data-dependent nature of the theoretical bound, which adjusts with the training process rather than remaining fixed. The converged value lies within the same order of magnitude as the manually selected separation margin (around $10^{-6}$ on PhysioNet12), indicating that the theoretical bound captures an appropriate scale of trajectory divergence without manual tuning. Compared with a fixed manually selected separation margin $\delta$, the data-dependent theoretically derived margin $\delta_{lb}$ enables more flexible trajectory separation during training and leads to improved classification performance.

\begin{figure}[t]
\centering
\includegraphics[width=0.9\linewidth,]{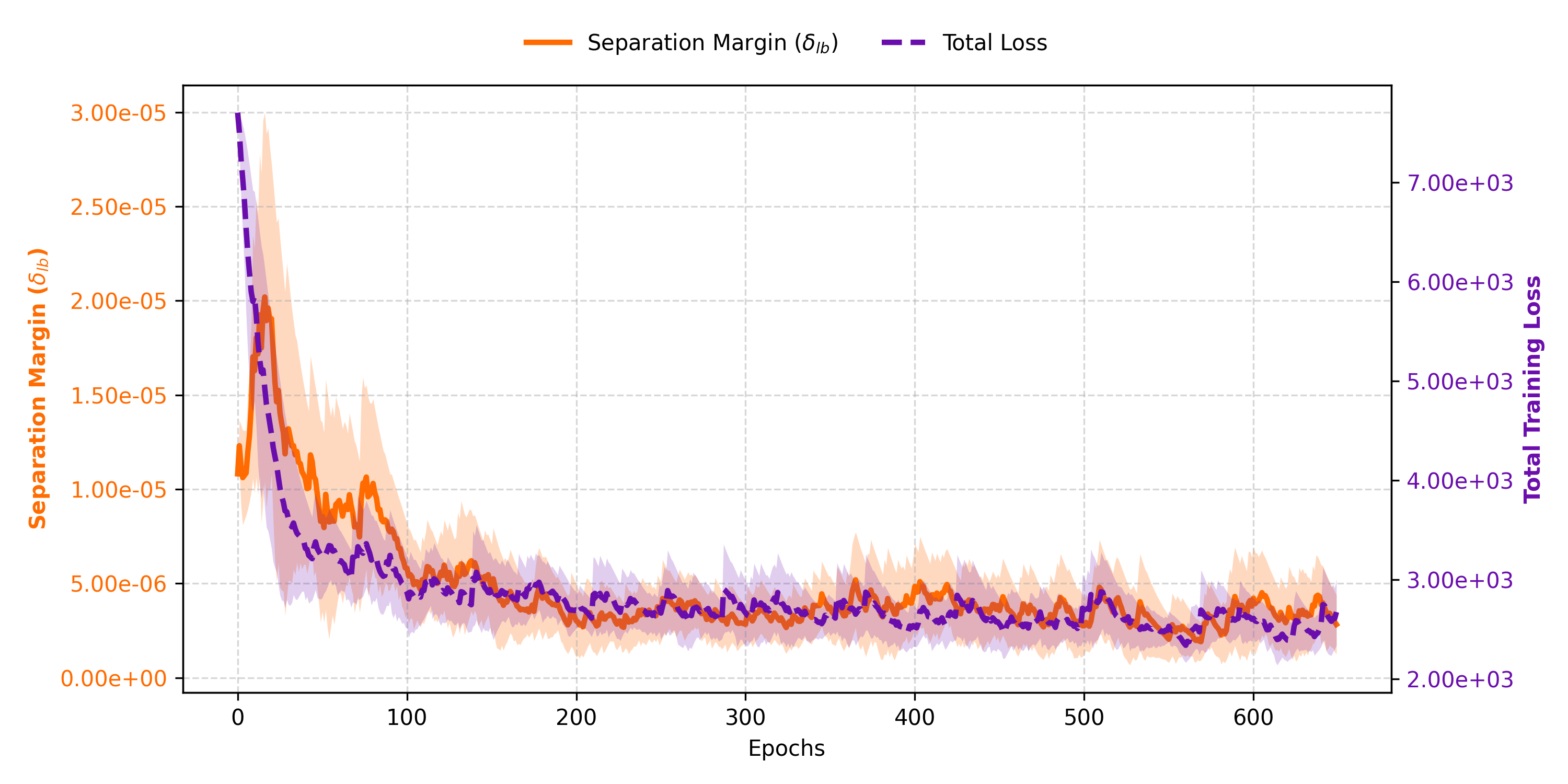}
\caption{Training dynamics of the theoretically derived separation margin $\delta_{lb}$ and total loss on PhysioNet12. The theoretically derived margin $\delta_{lb}$ maintains a comparable scale to the manually selected margin $\delta$, on the order of $10^{-6}$, across training.}
\label{fig:traj_evolution2}
\end{figure}

\section{Conclusion}
We propose GSNF, a one-step graph-structured Neural Flow that integrates an interaction graph into a parallel flow formulation for irregularly sampled multivariate time series. GSNF introduces trajectory-level self-supervision via ITG and RTG, which strengthens interaction learning by exploiting trajectory discrepancies in a one-step setting without iterative solvers. Experiments on five real-world datasets demonstrate that GSNF variants consistently achieve state-of-the-art performance, with GSNF($\delta_{\mathrm{lb}}$), which uses a theoretically derived separation margin, performing best in most cases. Among top-performing methods, GSNF attains this performance while requiring the lowest training time and peak GPU memory.

Future work will extend GSNF to time-varying interaction graphs, enabling dynamically evolving dependencies in non-stationary systems. We also plan to explore broader applications such as forecasting and generative modeling, and to examine its scalability in large-scale real-world scenarios.

\section{Impact Statements}
This paper presents work whose goal is to advance the field of machine learning. There are many potential societal consequences of our work, none of which we feel must be specifically highlighted here.

\nocite{langley00}

\bibliography{LaTeX/example_paper}
\bibliographystyle{icml2026}

%%%%%%%%%%%%%%%%%%%%%%%%%%%%%%%%%%%%%%%%%%%%%%%%%%%%%%%%%%%%%%%%%%%%%%%%%%%%%%%
%%%%%%%%%%%%%%%%%%%%%%%%%%%%%%%%%%%%%%%%%%%%%%%%%%%%%%%%%%%%%%%%%%%%%%%%%%%%%%%
% APPENDIX
%%%%%%%%%%%%%%%%%%%%%%%%%%%%%%%%%%%%%%%%%%%%%%%%%%%%%%%%%%%%%%%%%%%%%%%%%%%%%%%
%%%%%%%%%%%%%%%%%%%%%%%%%%%%%%%%%%%%%%%%%%%%%%%%%%%%%%%%%%%%%%%%%%%%%%%%%%%%%%%
\newpage
\appendix
\onecolumn
\input{supplementary.tex}

\end{document}

%% file: supplementary.tex
\clearpage

\section{Proof of Theorem \ref{thm:inveritbility} (Invertibility of Graph-Structured Neural Flows)}
\label{app:proof_invertibility}

\begin{proof}

Let $F(\mathbf z(t_0),t_0,t,A)$ denote the GSNF defined in Eq.~\eqref{eq:gsnf}.
We assume that the interaction function $g(\cdot,t_0,t,A)$ is contractive with
Lipschitz constant $L_g<1$, which is ensured by applying spectral normalization to all linear layers in both the MLP and GCN.

To establish invertibility, we show that $F(\cdot,t_0,t,A)$ is bi-Lipschitz.
For any two latent states $\mathbf z_1(t_0)$ and $\mathbf z_2(t_0)$, we have
\begin{equation}
\begin{aligned}
\|F(\mathbf{z}_1(t_0), t_0, t, A) - F(\mathbf{z}_2(t_0), t_0, t, A)\| 
&= \big\|\mathbf{z}_1(t_0) - \mathbf{z}_2(t_0) + \varphi(t-t_0)\, \big(g(\mathbf{z}_1(t_0), t_0, t, A) - g(\mathbf{z}_2(t_0), t_0, t, A)\big) \big\| \\
&\leq \|\mathbf{z}_1(t_0) - \mathbf{z}_2(t_0)\| + \varphi(t-t_0) \| g(\mathbf{z}_1(t_0), t_0, t, A) - g(\mathbf{z}_2(t_0), t_0, t, A) \| \\
&\leq (1 + \varphi(t-t_0) L_g) \|\mathbf{z}_1(t_0) - \mathbf{z}_2(t_0)\|.
\end{aligned}
\end{equation}

On the other hand, by the reverse triangle inequality,
\begin{equation}
\begin{aligned}
\|F(\mathbf{z}_1(t_0), t_0, t, A) - F(\mathbf{z}_2(t_0), t_0, t, A)\| 
&\geq \|\mathbf{z}_1(t_0) - \mathbf{z}_2(t_0)\| - \varphi(t-t_0) \| g(\mathbf{z}_1(t_0), t_0, t, A) - g(\mathbf{z}_2(t_0), t_0, t, A) \| \\
&\geq (1 - \varphi(t-t_0) L_g) \|\mathbf{z}_1(t_0) - \mathbf{z}_2(t_0)\|.
\end{aligned}
\end{equation}

Since $\varphi(t-t_0)\in[0,1)$ and $L_g<1$, we have $1-\varphi(t-t_0)L_g>0$.
Together with the upper and lower bounds derived above, this implies that
$F(\cdot,t_0,t,A)$ satisfies
\begin{equation}
0 < \bigl(1-\varphi(t-t_0)L_g\bigr)\,
\|\mathbf z_1(t_0)-\mathbf z_2(t_0)\|
\;\le\;
\|F(\mathbf z_1(t_0),t_0,t,A)-F(\mathbf z_2(t_0),t_0,t,A)\|
\;\le\;
\bigl(1+\varphi(t-t_0)L_g\bigr)\,
\|\mathbf z_1(t_0)-\mathbf z_2(t_0)\|,
\end{equation}
and is therefore bi-Lipschitz.

By standard properties of bi-Lipschitz mappings, its inverse $F^{-1}$ is
Lipschitz continuous with Lipschitz constant bounded by
$1/(1-\varphi(t-t_0)L_g)$.

\end{proof}

% \begin{theorem}[Invertibility of Graph-Structured Neural Flows]
% Given the graph-structured neural flow
% \begin{align}
% F(\mathbf{z}(t), t_0, t, A) &= \mathbf{z}(t) + \varphi(t)\cdot g(\mathbf{z}(t), t_0, t, A), \\
% g(\mathbf{z}(t), t_0, t, A) &= \mathrm{MLP}(\mathbf{z}(t)) \odot \mathrm{GCN}(\mathbf{z}(t), A),
% \end{align}
% where $\varphi(t_0) = 0$ and $\varphi(t) \in [0,1]$.  
% If spectral normalization is applied to each linear layer in both the MLP and the GCN,  
% then the mapping $F(\cdot, t_0, t, A)$ is invertible.
% \end{theorem}

% \begin{theorem}
% By spectral normalization, the operator norm of each linear layer is bounded, so the Lipschitz constant of $g$ satisfies $L_g < 1$.  

% With this property established, we proceed to analyze the Lipschitz continuity of $F$ and its inverse. 

% \end{theorem}

\section{Proof of Theorem \ref{thm:itg_lower_bound} (A Data-Dependent Lower Bound for the ITG Separation Margin)}
\label{app:itg_proof}

\begin{proof}
Recall that the GSNF adopts a residual flow parameterization
\begin{equation}
F(\mathbf z(t_0),t_0,t,A)
=
\mathbf z(t_0)
+
\varphi(t-t_0)\,g(\mathbf z(t_0),t_0,t,A).
\end{equation}

Let $\mathbf z(t)$ and $\mathbf z^*(t)$ denote the original and re-initialized latent
trajectories with initial states $\mathbf z_0$ and $\mathbf z_0^*$, respectively.
For $t\ge t_0^*$, define $\delta(t):=\mathbf z^*(t)-\mathbf z(t)$. Then
\begin{align}
\delta(t)
&=
F(\mathbf z_0^*,t_0^*,t,A)-F(\mathbf z_0,t_0,t,A)\nonumber\\
&=
(\mathbf z_0^*-\mathbf z_0)
+
\varphi(t-t_0^*)\,g(\mathbf z_0^*,t_0^*,t,A)
-
\varphi(t-t_0)\,g(\mathbf z_0,t_0,t,A).
\end{align}
Taking norms and applying the reverse triangle inequality yields
\begin{equation}
\label{eq:delta_ineq_clean}
\|\delta(t)\|
\ge
\Big\|
\varphi(t-t_0^*)\,g(\mathbf z_0^*,t_0^*,t,A)
-
\varphi(t-t_0)\,g(\mathbf z_0,t_0,t,A)
\Big\|
-
\|\mathbf z_0^*-\mathbf z_0\|.
\end{equation}
Denote $\Delta_{\mathrm{in}}:=\|\mathbf z_0^*-\mathbf z_0\|$.

We now lower-bound the residual term. Recall that
\[
g(\mathbf z(t_0),t_0,t,A)
=
\mathrm{MLP}(\mathbf z(t_0),t_0,t)
\odot
\mathrm{GCN}(\mathbf z(t_0),t_0,t,A),
\]
and assume the linear form
\(
\mathrm{GCN}(\mathbf z(t_0),t_0,t,A)=\mathcal A\,\mathbf z(t_0)\,W
\),
where $\mathcal A$ is the normalized adjacency matrix. All linear layers in both the MLP and the GCN are spectrally normalized, so that the MLP defines a Lipschitz-bounded element-wise modulation.

Then, for $t\ge t_0^*$, we define
\begin{equation}
\eta(t)
:=
\Big\|
\varphi(t-t_0^*)\,g(\mathbf z_0^*,t_0^*,t,A)
-
\varphi(t-t_0)\,g(\mathbf z_0,t_0,t,A)
\Big\|.
\end{equation}
By the spectral properties of the spectrally normalized MLP and the linear GCN,
$\eta(t)$ admits the lower bound
\begin{equation}
\eta(t)
\;\ge\;
\varphi(t-t_0^*)\,
\sigma_{\min}(\mathcal A)\,
\sigma_{\min}(W)\,
\|\mathbf z_0^*-\mathbf z_0\|.
\end{equation}

Substituting into Eq.~\eqref{eq:delta_ineq_clean} gives, for all $t\ge t_0^*$,
\begin{equation}
\|\delta(t)\|
\ge
\eta(t)-\Delta_{\mathrm{in}}
\ge
\Big(\sigma_{\min}(\mathcal A)\sigma_{\min}(W)-1\Big)\Delta_{\mathrm{in}}.
\end{equation}

Summing over discrete trajectory points $\{t_i\}_{i=k_0^*}^{L}$ yields
\begin{equation}
\sum_{i=k_0^*}^{L}\|\mathbf z^*(t_i)-\mathbf z(t_i)\|
\ge
\sum_{i=k_0^*}^{L}\bigl(\eta(t_i)-\Delta_{\mathrm{in}}\bigr)
\ge
(L-k_0^*+1)\bigl(\eta-\Delta_{\mathrm{in}}\bigr),
\end{equation}
where we define $\eta:=\sigma_{\min}(\mathcal A)\sigma_{\min}(W)\Delta_{\mathrm{in}}$.
Since the left-hand side is nonnegative, we finally obtain
\begin{equation}
\sum_{i=k_0^*}^{L}\|\mathbf z^*(t_i)-\mathbf z(t_i)\|
\ge
\max\!\Bigl\{0,\,(L-k_0^*+1)\bigl(\eta-\Delta_{\mathrm{in}}\bigr)\Bigr\},
\end{equation}
which completes the proof.
\end{proof}

\section{Algorithm}
We present the complete training procedure of GSNF in
Algorithm~\ref{alg:gsnf_training}, which summarizes graph inference,
initial state inference, forward trajectory generation, and trajectory-level regularization via ITG and RTG.

\begin{algorithm}[H]
   \caption{Training Procedure for GSNF}
   \label{alg:gsnf_training}
\begin{algorithmic}[1]
   \STATE {\bfseries Input:} Dataset $\mathcal{X}$; Labels $\mathcal{Y}$; Hyperparameters $\alpha, \beta, \gamma$; Re-initialization time $t_0^*$.
   \STATE {\bfseries Output:} Optimized parameters
   \WHILE{not converged}
       \STATE Sample batch $(X, y)$ from $\mathcal{X}$.
       \STATE \textit{// Stage 1. Interaction Graph Inference}:
       \STATE Compute segment-level adjacencies $A^{(s)}$ via self-attention;
       \STATE Aggregate global posterior $q_{\phi}(A|X) \leftarrow \sum w_s q_{\phi}(A^{(s)}|X^{(s)})$ (Eq. \ref{eq:posterior_A});
       \STATE Sample interaction graph $A \sim q_{\phi}(A|X)$;
       
       \STATE \textit{// Stage 2. Initial Latent State Inference}:
       \STATE Back-propagate observed states to $t_0$ via Inverse GSNF: $z_0^k \leftarrow F^{-1}(z(t_k), \dots)$;
       \STATE Aggregate posterior $q_{\phi}(z_0|X)$ and sample $z_0$ (Eq. \ref{eq:posterior_z});

       \STATE \textit{// Stage 3. Generation \& Classification}:
       \STATE Evolve trajectory $z(t) \leftarrow F(z_0, t_0, t, A)$ (Eq. \ref{eq: z-forward});
       \STATE Compute prediction $\hat{y}$ and losses $\mathcal{L}_{VAE}, \mathcal{L}_{CE}$;

       \STATE \textit{// Stage 4. Trajectory-Level  Self-Supervised Regularization}:
       \STATE \textbf{ITG:} Re-initialize at $t_0^*$ to get $z^*(t)$
       \STATE Determine separation margin $\delta$:
       \STATE \quad Option 1: Set $\delta$ (Manully Tuned Hyperparameter);
       \STATE \quad Option 2: Compute $\delta \leftarrow \delta_{lb}$ (Theoretical Lower Bound Eq.\ref{eq:lower_bound});
       \STATE Compute $\mathcal{L}_{ITG}$ using $\delta$ (Eq. \ref{eq:ITG});
       \STATE \textbf{RTG:} Reverse flow from $z(t_L)$ to get $\tilde{z}(t)$. Compute $\mathcal{L}_{RTG}$ (Eq. \ref{eq:RTG});
       
       \STATE \textit{// Stage 5. Update:}
       \STATE Minimize total loss $\mathcal{L}$ (Eq. \ref{eq:loss}).
   \ENDWHILE
\end{algorithmic}
\end{algorithm}

Among the components in Algorithm~\ref{alg:gsnf_training},
the ITG regularization step (Stage~4) relies on a separation margin~$\delta$,
whose choice is critical for effective trajectory separation.
Table~\ref{tab:delta} compares the manually tuned values of~$\delta$
used in our experiments with the corresponding theoretical lower bounds~$\delta_{lb}$
derived in Eq.~\ref{eq:lower_bound}.

\begin{table}[H] 
    \centering
    \begin{threeparttable}
        \small
        \setlength{\tabcolsep}{3.5pt}
        \renewcommand{\arraystretch}{1.15}
        \begin{tabular}{l c c c c c c}
            \toprule
            \textbf{ $\delta$($\times 10^{-6}$)}
            & \multicolumn{1}{c}{Physionet12}
            & \multicolumn{1}{c}{P12}
            & \multicolumn{1}{c}{P19}
            & \multicolumn{1}{c}{MIMIC-IV}
            & \multicolumn{1}{c}{eICU}\\
            \midrule
            Hyperparameter  & 1.00 &  10.00&  10.00& 0.10 &10.00\\
            \midrule
            Lower Bound     & 3.51 &  22.54& 15.76 & 0.47 &29.82 \\
            \bottomrule
        \end{tabular}
    \end{threeparttable}
    % --- 修正点：先写 Caption，再写 Label ---
    \caption{Comparison of manually selected separation margins ($\delta$) and the calculated theoretical lower bound ($\delta_{lb}$).} 
    \label{tab:delta} 
\end{table}

% \section{Degeneration to Neural Flow}
% We demonstrate that GSNF generalizes ResNet Flow ~\cite{bilovs2021neural}. Recall the GSNF dynamics defined in Eq. \ref{eq:gsnf}:
% \begin{align*}
% F(\mathbf{z}(t), t_0, t, A) &= \mathbf{z}(t) + \varphi(t)\cdot g(\mathbf{z}(t), t_0, t, A), \\
% g(\mathbf{z}(t), t_0, t, A) &= \mathrm{MLP}(\mathbf{z}(t), t_0, t) \odot \mathrm{GCN}(\mathbf{z}(t), t_0, t,  A),
% \label{eq:gsnf}
% \end{align*}

% In the non-interacting case where $A$ is the identity matrix $I$, the graph convolution operation, typically defined as $h_i' = \sigma(\sum_{j \in \mathcal{N}(i)} c_{ij} W h_j)$, simplifies significantly. Since the neighborhood $\mathcal{N}(i)$ contains only the node itself, the aggregation collapses to a simple node-wise projection: $\text{GCN}(z, I) = \sigma(Wz)$.

% Consequently, the interaction term reduces to a component-wise function $\tilde{g}(\mathbf{z}(t), t_0,t) = \mathrm{MLP}(\mathbf{z},t_0,t) \odot \sigma(W\mathbf{z}(t))$, which acts independently on each variable. The dynamics thus become $F(\mathbf{z}(t), t_0,t) = \mathbf{z}(t) + \varphi(t)\tilde{g}(\mathbf{z}, t_0, t)$, exactly recovering the ResNet Flow formulation. The spectral normalization applied in GSNF ensures $\text{Lip}(\tilde{g}) < 1$, guaranteeing invertibility as in the original method.

\section{Comprehensive Experiments}
%This section presents further experimental results, including performance analysis, grouped bar plots, detailed ablation studies and evaluations on synthetic datasets.

\subsection{Memory Usage and Training Time}
\label{app:performance_analysis}
Table ~\ref{tab:efficiency_comparison} presents the classification performance and computational efficiency of the proposed GSNF($\delta_{lb}$) variant compared to state-of-the-art baselines on the PhysioNet12 dataset. In terms of training speed, GSNF requires 9 seconds per epoch, making it the fastest among the top-performing models. Compared to high-accuracy baselines like DualDynamics (37s/epoch) and FlowPath (17s/epoch), GSNF achieves a speedup of approximately 2$\times$ to 4$\times$. This efficiency is attributed to the one-step flow architecture, which avoids the iterative integration required by ODE-based solvers. Regarding memory usage, GSNF consumes 6349 MB, which is comparable to other graph-based approaches.

\begin{table}[H] % 如果需要跨双栏请改为 table*
\centering
\begin{threeparttable}
\small
\setlength{\tabcolsep}{8pt} % 稍微增加列间距让表格更美观
\renewcommand{\arraystretch}{1.15}
\begin{tabular}{l c c c c}
\toprule
Method & AUROC & AUPRC & Time (s) & Memory (MB) \\
\midrule
GRU-D           & 79.1 & 42.7 & 25 & \underline{2456} \\
ODE-RNN         & 80.8 & 33.7 & 45 & 3407 \\
NeuralFlow      & 80.9 & 51.5 & 30 & 4999 \\
IVP-VAE         & 81.1 & 46.2 & \textbf{3}  & 5363 \\
DualDynamics    & 86.1 & 55.3 & 37 & 6743 \\
FlowPath        & 85.3 & 55.3 & 17 & 7356 \\
\midrule
RainDrop        & 81.2 & 37.3 & 28 & 3200 \\
GraphNeuralFlow & 84.5 & 53.7 & 35 & 6056 \\
Hi-Patch        & \underline{86.4} & \underline{56.5} & 24 & 6577 \\
\midrule
mTAN            & 85.8 & 50.4 & 12 & \textbf{1857} \\
Warpformer      & 83.4 & 43.5 & 40 & 4532 \\
TimeCHEAT       & 84.5 & 46.3 & 35 & 6859 \\
ViTST           & 81.3 & 37.4 & 59 & 8982 \\
\midrule
\textbf{GSNF (Ours)} & \textbf{86.7} & \textbf{56.9} & \underline{9} & 6349 \\
\bottomrule
\end{tabular}
\end{threeparttable}
\caption{Performance and efficiency comparison on PhysioNet12. \textbf{Bold} and \underline{underlined} mark the best and second-best results, respectively.}
\label{tab:efficiency_comparison}
\end{table}

\subsection{Main Result}

Fig. \ref{fig:mainResult} complements the main text by visualizing the performance variability across five independent runs. The error bars, representing standard deviation, indicate that GSNF maintains consistently low variance across all datasets, demonstrating superior stability compared to baselines like Raindrop or FlowPath which exhibit higher volatility in certain metrics. Furthermore, the theoretical variant GSNF($\delta_{lb}$) shows comparable stability to the manually tuned GSNF($\delta$), confirming that using the calculated lower bound yields robust results without compromising training stability.

\begin{figure}
\centering
\includegraphics[width=1\linewidth,]{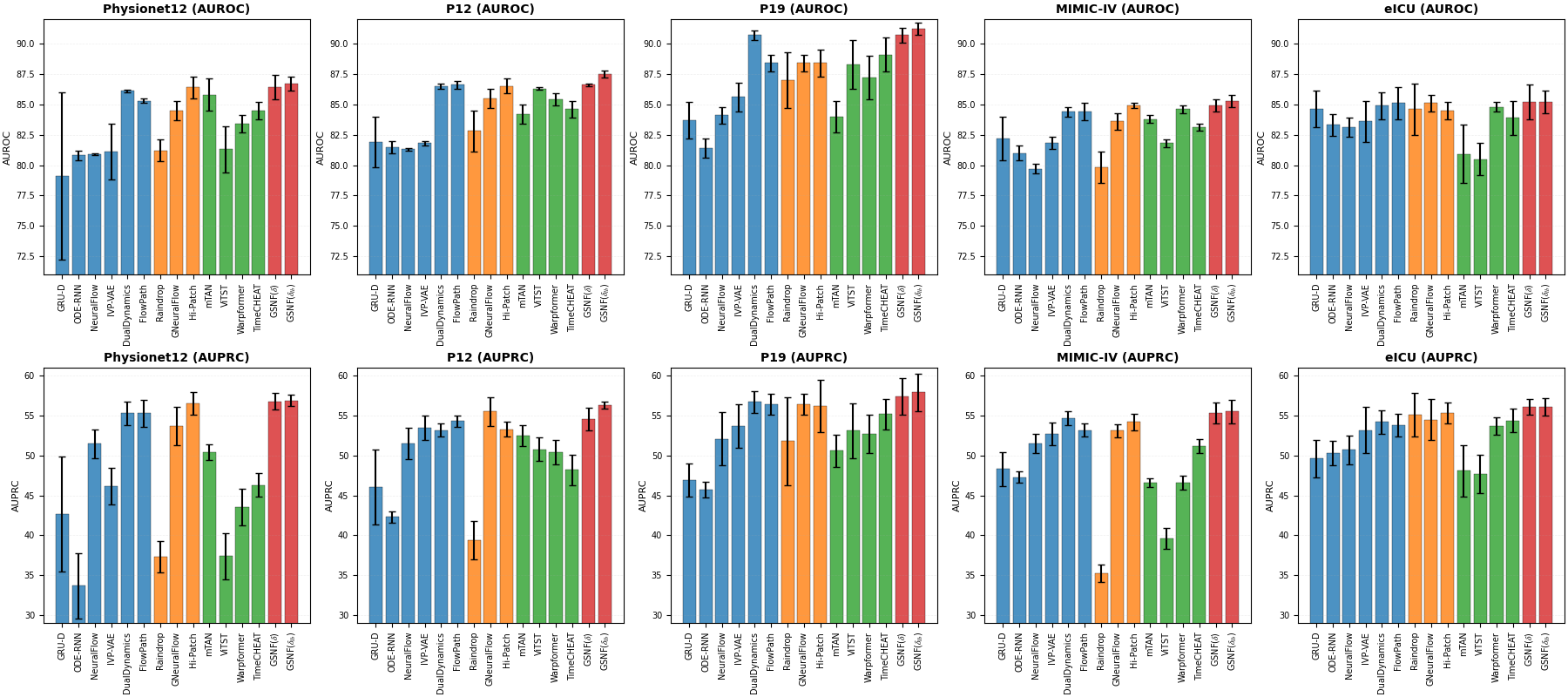}
\caption{Main Result}
\label{fig:mainResult}
\end{figure}

\section{Experimental Protocol}
\label{app:experiment}

We conduct comparative experiments using five-fold cross-validation to evaluate GSNF against several representative methods, with AUROC and AUPRC as the primary metrics,

The training workflow of entire model is summarized in Algorithm~\ref{alg:gsnf_training}. For testing, the model bypasses loss calculation and directly outputs the predicted labels via a forward pass.
All five datasets are used for classification experiments. Each dataset is randomly split into 80\% for training, 10\% for validation, and 10\% for testing. Following previous works~\cite{rubanova2019latent,shukla2021multi,zhang2021graph}, we repeat each experiment five times using different random seeds to split the datasets and initialize model parameters. For classification experiments, we focus on predicting in-hospital mortality using the first 48 hours of data. Due to class imbalance in these datasets, we assess classification performance using the area under the ROC curve (AUROC) and the area under the precision-recall curve (AUPRC). All models were tested in the same computing environment. The details are as follows:
\setlist[itemize,1]{leftmargin=1em,itemsep=0pt}
\begin{itemize}
\item Operating System: Ubuntu 22.04.1 LTS
\item CPU: Intel(R) Xeon(R) Gold 6330 CPU @ 2.00GHz
\item GPU: NVIDIA GeForce RTX 3090 with 24 GB of memory
\end{itemize}
%In the following, we provide a detailed description of the hyperparameter settings used in our experiments.

\subsection{Hyperparameters}
\label{app:hyperparams}
For reproducibility, we summarize the hyperparameter settings shared across all experiments,
together with the GSNF-specific configurations, in Table~\ref{tab:hparams_all}.

\begin{table}[H]
\centering
\small
\begin{tabular}{l c c}
\toprule
\textbf{Hyperparameter} & \textbf{Value} & \textbf{Scope} \\
\midrule
Optimizer & Adam & All \\
Weight decay & $1\times10^{-4}$ & All \\
Batch size & 50 & All \\
Learning rate (LR) & $1\times10^{-3}$ & All \\
LR scheduler step & 20 & All \\
LR decay factor & 0.5 & All \\
\midrule
Number of GSNF layers & 2 & GSNF \\
Latent dimension & Number of sensors & GSNF \\
Hidden layers & 3 & GSNF \\
Hidden dimension & 128 & GSNF \\
Cross-entropy weight $\alpha$ & 1000 & GSNF \\
ITG weight $\beta$ & 0.1 & GSNF \\
RTG weight $\gamma$ & 0.1 & GSNF \\
\bottomrule
\end{tabular}
\caption{Hyperparameter settings for all experiments and GSNF.}
\label{tab:hparams_all}
\end{table}

% \subsubsection{All experiments.}
% These are the shared hyperparameters applied across all tasks.
% \begin{itemize}
% \item Optimizer: Adam
% \item Weight decay: 1e-4
% \item Batch size: 50
% \item Learning rate: 1e-3
% \item Learning rate scheduler step: 20
% \item Learning rate decay: 0.5
% \end{itemize}
% \subsubsection{Detailed settings of GSNF.}
% Specific hyperparameters used to configure GSNF for optimal performance.
% \begin{itemize}
% \item Layers of Graph-Structured Neural Flow: 2
% \item Latent state dimension: number of sensors
% \item Hidden layers: 3
% \item Hidden dimension: 128
% \item Weight $\alpha$ for cross-entropy: 1000
% \item Weight $\beta$ for ITG: 0.1
% \item Weight $\gamma$ for RTG: 0.1
% \end{itemize}

\subsection{Datasets  and Preprocessing}
% \subsection{Synthetic data}
% We generate four synthetic systems with the graph size varying from 3 to 30. These systems follow the graph-based equation (4) or solution (5). They are named “Sink,” “Triangle,” “Sawtooth,” and “Square,” following those defined in [4]; but we add a graph to make the system SEM-like.

Our model was evaluated on five representative medical datasets featuring irregularly sampled time series. For consistency, the preprocessing procedures were adopted following the respective referenced works. Detailed information on each dataset and its preprocessing steps is provided below, and the key statistics of the processed datasets are summarized in Table \ref{tab:dataset_information}.

\begin{table}[H]
\centering
\label{tab:main2}
\begin{threeparttable}
\small
\setlength{\tabcolsep}{3pt}
\renewcommand{\arraystretch}{1.15}
% \resizebox{0.9\columnwidth}{!}{
\begin{tabular}{l c c c c c}
\toprule
& \multicolumn{1}{c}{Physionet12}
& \multicolumn{1}{c}{P12}
& \multicolumn{1}{c}{P19}
& \multicolumn{1}{c}{MIMIC-IV}
& \multicolumn{1}{c}{eICU} \\
\midrule
\#Samples      & 3,989 & 11,988  & 38,803 & 26,070 & 12,312\\
\#Variables   & 37 & 36 & 39  & 96 & 14 \\
Missing ratio (\%)   & 84.34 & 88.4 &94.9 & 97.95 & 65.25\\
Positive rate (\%) & 13.89 & 7 & 4 & 13.39 & 17.61\\
\bottomrule
\end{tabular}
% }
\end{threeparttable}
\caption{Key information of the five datasets.} 
%after preprocessing: Number of admissions used, number of selected variables, overall percentage of missing values, positive rate for mortality.}
\label{tab:dataset_information}
\end{table}
%we evaluated our model on four datasets in medical field where irregular sampled time series is most widely used, namely P12~\cite{goldberger2000physiobank},PhysioNet12~\cite{silva2012predicting},eICU~\cite{pollard2018eicu} and MIMIC-IV~\cite{johnson2020mimic}. 

The \textbf{PhysioNet 2012} dataset~\cite{silva2012predicting}was released for the PhysioNet/Computing in Cardiology Challenge 2012, aiming to predict in-hospital mortality which is a reduced
version of P12 considered by prior work. It contains patient information from ICU admissions, including vital signs, lab tests, and demographic data. In our experiments, we follow Neural Flow~\cite{bilovs2021neural} and utilize the 4,000 admissions from the challenge’s training set, focusing on 37 features recorded within the first 48 hours of each patient’s stay.

The \textbf{P12} dataset~\cite{goldberger2000physiobank}  comprises data from 11,988 patients, including 36 sensor variables and a binary label indicating survival during hospitalization. We used the processed data provided by Raindrop~\cite{zhang2021graph}.

The \textbf{P19} dataset~\cite{reyna2020early} was released for the PhysioNet/Computing in Cardiology Challenge 2019, aiming to predict the onset of sepsis. It contains patient information from ICU stays, comprising static demographics and sparse time-dependent physiological measurements. In our experiments, we utilize 38803 variable-length time series, focusing on 39 features (5 static and 34 time-dependent) recorded within the first 72 hours of each patient’s stay for the binary classification of sepsis development.

The \textbf{MIMIC-IV} dataset~\cite{johnson2020mimic}  is a multivariate time series dataset composed of sparse and irregularly sampled physiological data collected at the Beth Israel Deaconess Medical Center between 2008 and 2019. Following a preprocessing approach similar to that of Neural Flow~\cite{bilovs2021neural}, we extract 96 features — including patient intake/output, lab results, and medication prescriptions — from the first 48 hours post-ICU admission. A total of 26,070 patient stays are retained for use in classification tasks.

The \textbf{eICU} Collaborative Research Database~\cite{pollard2018eicu} contains data from patients admitted to ICUs across 208 hospitals in the United States between 2014 and 2015. Following the preprocessing steps outlined by IVP-VAE~\cite{xiao2024ivp}, we extract 14 features within the initial 48 hours post-ICU admission from a total of 12,312 patient stays.

\subsection{Baselines}
We compare our model against several baselines for the classification of multivariate irregular time-series.
\setlist[itemize,1]{leftmargin=1em,itemsep=0pt}

\begin{itemize}
\item \textbf{Continuous-time model}:
\begin{itemize}
\item \textbf{GRU-D}~\cite{che2018recurrent} incorporates missing patterns using GRU combined with a learnable decay mechanism on both the input sequence and hidden states.
\item \textbf{ODE-RNN}~\cite{rubanova2019latent} uses an ODE-RNN encoder and Neural ODE decoder in a VAE architecture.
\item \textbf{NeuralFlow}~\cite{bilovs2021neural} model the solution curves directly, with a neural network, instead of specifying the derivative. 
\item \textbf{IVP-VAE}~\cite{xiao2024ivp}  models irregular time series using a single invertible IVP-based continuous process, eliminating recurrent components and enabling parallel state evolution.
\item \textbf{DualDynamics}~\cite{oh2025dualdynamics} combines NDE-based method and Neural Flowbased method enhances expressive power.
\item \textbf{FlowPath}~\cite{oh2025flowpath} employs an invertible neural flow to learn the geometry of the control path, leveraging invertibility constraints to construct a continuous and data-adaptive manifold for robust modeling of sparse and irregularly-sampled time series.
\end{itemize}

\item \textbf{Graph-based models}:
\begin{itemize}
\item \textbf{Raindrop}~\cite{zhang2021graph} represents dependencies among multivariates with a graph whose connectivity is learned from time series.
\item \textbf{GNeuralFlow}~\cite{mercatali2024graph} using a directed acyclic graph to model the conditional dependencies of the system components and learning this graph in tandem with neural flow.
\item \textbf{Hi-Patch}~\cite{luohi} integrates intra-patch graphs for densely sampled local modeling and inter-patch graphs for global multi-scale analysis, leveraging a hierarchical architecture to handle variables with distinct origin scales in Irregular Multivariate Time Series.
\end{itemize}

\item \textbf{Other strong baselines}:
\begin{itemize}
\item \textbf{mTAN}~\cite{shukla2021multi} leverages an attention mechanism to learn temporal similarity and time embeddings.
\item \textbf{ViTST}~\cite{li2023time} transforms irregularly sampled time series into line graph images and applies pre-trained vision transformers for classification.
\item \textbf{Warpformer}~\cite{zhang2023warpformer} addresses intra-series irregularity and inter-series discrepancy in irregular time series by introducing a warping-based architecture with specialized input encoding.
\item \textbf{TimeCHEAT}~\cite{liu2025timecheat} combine channel-dependent modeling at the local (sub-series) level and channel-independent attention at the global level, leveraging bipartite graph-based embedding and Transformer architecture.
\end{itemize}
\end{itemize}

%% file: LaTeX/example_paper.bib
@article{kingma2013auto,
  title={Auto-encoding variational bayes},
  author={Kingma, Diederik P and Welling, Max},
  journal={arXiv preprint arXiv:1312.6114},
  year={2013}
}

@article{chen2018neural,
  title={Neural ordinary differential equations},
  author={Chen, Ricky TQ and Rubanova, Yulia and Bettencourt, Jesse and Duvenaud, David K},
  journal={Advances in neural information processing systems},
  volume={31},
  year={2018}
}

@article{poli2019graph,
  title={Graph neural ordinary differential equations},
  author={Poli, Michael and Massaroli, Stefano and Park, Junyoung and Yamashita, Atsushi and Asama, Hajime and Park, Jinkyoo},
  journal={Proceedings of the AAAI Conference on Artificial Intelligence},
  year={2019}
}

@inproceedings{huang2023generalizing,
  title={Generalizing graph ode for learning complex system dynamics across environments},
  author={Huang, Zijie and Sun, Yizhou and Wang, Wei},
  booktitle={Proceedings of the 29th ACM SIGKDD Conference on Knowledge Discovery and Data Mining},
  pages={798--809},
  year={2023}
}

@article{huang2020learning,
  title={Learning continuous system dynamics from irregularly-sampled partial observations},
  author={Huang, Zijie and Sun, Yizhou and Wang, Wei},
  journal={Advances in Neural Information Processing Systems},
  volume={33},
  pages={16177--16187},
  year={2020}
}

@inproceedings{huang2021coupled,
  title={Coupled graph ode for learning interacting system dynamics},
  author={Huang, Zijie and Sun, Yizhou and Wang, Wei},
  booktitle={Proceedings of the 27th ACM SIGKDD conference on knowledge discovery \& data mining},
  pages={705--715},
  year={2021}
}

@article{rubanova2019latent,
  title={Latent ordinary differential equations for irregularly-sampled time series},
  author={Rubanova, Yulia and Chen, Ricky TQ and Duvenaud, David K},
  journal={Advances in neural information processing systems},
  volume={32},
  year={2019}
}

@article{de2019gru,
  title={GRU-ODE-Bayes: Continuous modeling of sporadically-observed time series},
  author={De Brouwer, Edward and Simm, Jaak and Arany, Adam and Moreau, Yves},
  journal={Advances in neural information processing systems},
  volume={32},
  year={2019}
}

@article{chen2023contiformer,
  title={Contiformer: Continuous-time transformer for irregular time series modeling},
  author={Chen, Yuqi and Ren, Kan and Wang, Yansen and Fang, Yuchen and Sun, Weiwei and Li, Dongsheng},
  journal={Advances in Neural Information Processing Systems},
  volume={36},
  pages={47143--47175},
  year={2023}
}

@article{kidger2020neural,
  title={Neural controlled differential equations for irregular time series},
  author={Kidger, Patrick and Morrill, James and Foster, James and Lyons, Terry},
  journal={Advances in neural information processing systems},
  volume={33},
  pages={6696--6707},
  year={2020}
}

@article{kidger2021efficient,
  title={Efficient and accurate gradients for neural sdes},
  author={Kidger, Patrick and Foster, James and Li, Xuechen Chen and Lyons, Terry},
  journal={Advances in Neural Information Processing Systems},
  volume={34},
  pages={18747--18761},
  year={2021}
}

@article{bilovs2021neural,
  title={Neural flows: Efficient alternative to neural ODEs},
  author={Bilo{\v{s}}, Marin and Sommer, Johanna and Rangapuram, Syama Sundar and Januschowski, Tim and G{\"u}nnemann, Stephan},
  journal={Advances in neural information processing systems},
  volume={34},
  pages={21325--21337},
  year={2021}
}

@inproceedings{xiao2024ivp,
  title={IVP-VAE: modeling EHR time series with initial value problem solvers},
  author={Xiao, Jingge and Basso, Leonie and Nejdl, Wolfgang and Ganguly, Niloy and Sikdar, Sandipan},
  booktitle={Proceedings of the AAAI Conference on Artificial Intelligence},
  volume={38},
  number={14},
  pages={16023--16031},
  year={2024}
}

@article{oh2025comprehensive,
  title={Comprehensive review of neural differential equations for time series analysis},
  author={Oh, YongKyung and Kam, Seungsu and Lee, Jonghun and Lim, Dong-Young and Kim, Sungil and Bui, Alex},
  journal={the Thirty-Fourth International Joint Conference on Artificial Intelligence},
  year={2025}
}

@article{mercatali2024graph,
  title={Graph neural flows for unveiling systemic interactions among irregularly sampled time series},
  author={Mercatali, Giangiacomo and Freitas, Andre and Chen, Jie},
  journal={Advances in Neural Information Processing Systems},
  volume={37},
  pages={57183--57206},
  year={2024}
}

@article{goldberger2000physiobank,
  title={PhysioBank, PhysioToolkit, and PhysioNet: components of a new research resource for complex physiologic signals},
  author={Goldberger, Ary L and Amaral, Luis AN and Glass, Leon and Hausdorff, Jeffrey M and Ivanov, Plamen Ch and Mark, Roger G and Mietus, Joseph E and Moody, George B and Peng, Chung-Kang and Stanley, H Eugene},
  journal={circulation},
  volume={101},
  number={23},
  pages={e215--e220},
  year={2000},
  publisher={Lippincott Williams \& Wilkins}
}

@inproceedings{silva2012predicting,
  title={Predicting in-hospital mortality of icu patients: The physionet/computing in cardiology challenge 2012},
  author={Silva, Ikaro and Moody, George and Scott, Daniel J and Celi, Leo A and Mark, Roger G},
  booktitle={2012 computing in cardiology},
  pages={245--248},
  year={2012},
  organization={IEEE}
}

@article{pollard2018eicu,
  title={The eICU Collaborative Research Database, a freely available multi-center database for critical care research},
  author={Pollard, Tom J and Johnson, Alistair EW and Raffa, Jesse D and Celi, Leo A and Mark, Roger G and Badawi, Omar},
  journal={Scientific data},
  volume={5},
  number={1},
  pages={1--13},
  year={2018},
  publisher={Nature Publishing Group}
}

@article{johnson2020mimic,
  title={Mimic-iv},
  author={Johnson, Alistair and Bulgarelli, Lucas and Pollard, Tom and Horng, Steven and Celi, Leo Anthony and Mark, Roger},
  journal={PhysioNet. Available online at: https://physionet. org/content/mimiciv/1.0/(accessed August 23, 2021)},
  pages={49--55},
  year={2020}
}

@article{zhang2021graph,
  title={Graph-guided network for irregularly sampled multivariate time series},
  author={Zhang, Xiang and Zeman, Marko and Tsiligkaridis, Theodoros and Zitnik, Marinka},
  journal={International Conference on Learning Representations},
  year={2022}
}

@article{che2018recurrent,
  title={Recurrent neural networks for multivariate time series with missing values},
  author={Che, Zhengping and Purushotham, Sanjay and Cho, Kyunghyun and Sontag, David and Liu, Yan},
  journal={Scientific reports},
  volume={8},
  number={1},
  pages={6085},
  year={2018},
  publisher={Nature Publishing Group UK London}
}

@inproceedings{oh2025dualdynamics,
  title={Dualdynamics: Synergizing implicit and explicit methods for robust irregular time series analysis},
  author={Oh, YongKyung and Lim, Dong-Young and Kim, Sungil},
  booktitle={Proceedings of the AAAI Conference on Artificial Intelligence},
  volume={39},
  number={18},
  pages={19730--19739},
  year={2025}
}

@article{shukla2021multi,
  title={Multi-time attention networks for irregularly sampled time series},
  author={Shukla, Satya Narayan and Marlin, Benjamin M},
  journal={International Conference on Learning Representations},
  year={2021}
}

@article{li2023time,
  title={Time series as images: Vision transformer for irregularly sampled time series},
  author={Li, Zekun and Li, Shiyang and Yan, Xifeng},
  journal={Advances in Neural Information Processing Systems},
  volume={36},
  pages={49187--49204},
  year={2023}
}

@inproceedings{zhang2023warpformer,
  title={Warpformer: A multi-scale modeling approach for irregular clinical time series},
  author={Zhang, Jiawen and Zheng, Shun and Cao, Wei and Bian, Jiang and Li, Jia},
  booktitle={Proceedings of the 29th ACM SIGKDD Conference on Knowledge Discovery and Data Mining},
  pages={3273--3285},
  year={2023}
}

@inproceedings{liu2025timecheat,
  title={Timecheat: A channel harmony strategy for irregularly sampled multivariate time series analysis},
  author={Liu, Jiexi and Cao, Meng and Chen, Songcan},
  booktitle={Proceedings of the AAAI Conference on Artificial Intelligence},
  volume={39},
  number={18},
  pages={18861--18869},
  year={2025}
}

@inproceedings{zhou2023detecting,
  title={Detecting multivariate time series anomalies with zero known label},
  author={Zhou, Qihang and Chen, Jiming and Liu, Haoyu and He, Shibo and Meng, Wenchao},
  booktitle={Proceedings of the AAAI Conference on Artificial Intelligence},
  volume={37},
  number={4},
  pages={4963--4971},
  year={2023}
}

@inproceedings{luohi,
  title={Hi-Patch: Hierarchical Patch GNN for Irregular Multivariate Time Series},
  author={Luo, Yicheng and Zhang, Bowen and Liu, Zhen and Ma, Qianli},
  booktitle={Forty-second International Conference on Machine Learning},
  year={2025}
}

@article{oh2025flowpath,
  title={FlowPath: Learning Data-Driven Manifolds with Invertible Flows for Robust Irregularly-sampled Time Series Classification},
  author={Oh, YongKyung and Lim, Dong-Young and Kim, Sungil},
  journal={Proceedings of the AAAI Conference on Artificial
Intelligence},
  year={2026}
}

@article{reyna2020early,
  title={Early prediction of sepsis from clinical data: the PhysioNet/Computing in Cardiology Challenge 2019},
  author={Reyna, Matthew A and Josef, Christopher S and Jeter, Russell and Shashikumar, Supreeth P and Westover, M Brandon and Nemati, Shamim and Clifford, Gari D and Sharma, Ashish},
  journal={Critical care medicine},
  volume={48},
  number={2},
  pages={210--217},
  year={2020},
  publisher={LWW}
}
